\documentclass[runningheads]{llncs}

\usepackage[a4paper,margin=1in]{geometry} 

\usepackage{graphicx}

\usepackage{cite}
\usepackage{units}

\usepackage{amsmath,amssymb,amsfonts,amsthm}

\usepackage{bm}
\usepackage{graphicx}
\usepackage{textcomp}
\usepackage{xcolor}
\usepackage{enumerate}
\usepackage{braket}
\usepackage{marvosym}
\def\BibTeX{{\rm B\kern-.05em{\sc i\kern-.025em b}\kern-.08em
    T\kern-.1667em\lower.7ex\hbox{E}\kern-.125emX}}
\usepackage[linesnumbered,ruled,vlined,algo2e]{algorithm2e}
\newcommand{\myref}[1]{Eq. \eqref{#1}}
\usepackage{subfigure}

\usepackage{svg}
\theoremstyle{definition}
\newtheorem{assumption}{Assumption}

\SetKwFor{For}{for}{do}{endfor}
\allowdisplaybreaks[4]
\DeclareMathOperator*{\argmin}{arg\,min}
\begin{document}
\title{Adaptive pruning-based Newton's method for distributed learning\thanks{Partialy supported by National Natural Science Foundation of China under Grant numbers 62302247 and 623B2068, Shandong Province Natural Science Foundation under Grant numbers ZR20220F140, Postdoctoral Fellowship Program of CPSF under Grant numbers:GZC20231460, China Scholarship Council under Grant numbers 202306220153 and China Postdoctoral Science Foundation
under Grant numbers 2024M761806. \textit{Corresponding author: Yuan Yuan.Email address:yyuan@sdu.edu.cn}}}
\author{Shuzhen Chen\inst{1} \and
Yuan Yuan\inst{2} \and
Youming Tao\inst{3,5} \and Tianzhu Wang\inst{3} \and Zhipeng Cai\inst{4} \and Dongxiao Yu\inst{3}
\institute{College of Computer Science and Technology, Ocean University of China, Qingdao 266100, China\\
\and
School of Software \& Joint SDU-NTU Centre for Artificial Intelligence Research (C-FAIR), Shandong University, Jinan 250100, China.
\and
School of Computer Science and Technology, Shandong University, Qingdao 266237, China
\and
Department of Computer Science, Georgia State University, Atlanta, GA 30303, USA.
\and
School of Electrical Engineering and Computer Science, TU Berlin, Berlin 10587, Germany.\\
}}
\maketitle      
\begin{abstract}
Newton's method leverages curvature information to boost performance, and
thus outperforms first-order methods for distributed learning problems.
However, Newton's method is not practical in large-scale and heterogeneous
learning environments, due to obstacles such as high computation and communication
costs of the Hessian matrix, sub-model diversity, staleness of training,
and data heterogeneity. To overcome these obstacles, this paper presents
a novel and efficient algorithm named Distributed Adaptive Newton Learning
(\texttt{DANL}), which solves the drawbacks of Newton's method by using
a simple Hessian initialization and adaptive allocation of training regions.
The algorithm exhibits remarkable convergence properties, which are rigorously
examined under standard assumptions in stochastic optimization. The theoretical
analysis proves that \texttt{DANL} attains a linear convergence rate while
efficiently adapting to available resources and keeping high efficiency.
Furthermore, \texttt{DANL} shows notable independence from the condition
number of the problem and removes the necessity for complex parameter tuning.
Experiments demonstrate that \texttt{DANL} achieves linear
convergence with efficient communication and strong performance across
different datasets.

\end{abstract}

\section{Introduction}\label{sec1}
Distributed machine learning (DML) is a paradigm that leverages the parallelism
and scalability of modern computing architectures to process and analyze
massive and complex data. DML encompasses a variety of techniques that
enable the collaborative training and inference of models across multiple
computing nodes (a.k.a. workers). One of the fundamental
research problems in DML is distributed stochastic optimization (DSO),
which aims to find the optimal model parameters by minimizing a loss function
over stochastic data. DSO is essential for various DML applications, such
as federated learning, distributed deep learning, and distributed reinforcement
learning. These applications have practical implications, such as facilitating
privacy-preserving and cooperative learning among multiple parties, enhancing
the performance and robustness of deep neural networks, and solving complex
and dynamic decision-making problems.

DSO problems can be tackled by two classes of methods: first-order and
second-order. First-order methods, such as stochastic gradient descent
(SGD), rely only on the gradient of the loss function, which indicates
the direction of the steepest descent. They are simple and scalable, but
they suffer from some limitations. They have slow convergence rates and
depend heavily on the step size and the condition number of the loss function~\cite{21DBLP:conf/aistats/ChenYGG22,8DBLP:conf/icml/IslamovQR21,7DBLP:conf/cdc/0001YB20,28DBLP:books/daglib/0034927}.
They also require frequent communication of gradients among the computing
workers, which incurs high communication overhead. Second-order
methods, such as Newton's method, utilize the Hessian of the loss function,
which is the matrix of the second derivatives. They have several advantages
over first-order methods. They use extra curvature information to achieve
faster and more precise convergence. They can also handle more complex
aspects of the loss function's geometry by using second-order derivatives,
which enables them to optimize effectively in complicated landscapes. Moreover,
second-order methods can reduce the communication rounds needed in DSO
scenarios since they usually converge in fewer iterations than first-order
methods \cite{agarwal2017second}.

Newton's method is a second-order method that outperforms first-order methods,
but it is challenging to apply it to DSO, especially in realistic large-scale
and heterogeneous learning environments. The main challenge is the resource
constraints of computation, storage, and communication. Computationally,
the Hessian matrix of the loss function is expensive to invert, store,
and transmit for large-scale problems \cite{byrd2016stochastic}. The problem
is more severe for high-dimensional problems, which need a lot of storage
and manipulation for the huge and dense Hessian matrix. Moreover, the communication
overhead of sending the matrix among distributed workers and servers is
a big bottleneck in DML. Resource constraints also affect the learning
process, as the computing workers have different capabilities
and limitations, such as processing speeds, memory sizes, and battery lives.
In realistic scenarios, computing workers often only access
or compute on a subset of model parameters, creating sub-models. The heterogeneity
and discrepancy among workers and data sources cause noise,
bias, and variance, which influence the convergence and accuracy of DML
\cite{DBLP:journals/ftml/KairouzMABBBBCC21}. Considering the challenges
above, we investigate this question:
\textit{can Newton's method effectively solve distributed stochastic optimization
in heterogeneous environments?} However, applying Newton's method to distributed
stochastic optimization faces several challenges.

\noindent
(1).
\textit{Hessian computation and communication:} The high
cost of computing and transmitting the Hessian matrix across distributed
nodes makes it impractical in large-scale settings.

\noindent
(2).
\textit{Sub-model diversity:} Sub-model diversity partitions
a model into independently trained regions, allowing workers to allocate
resources adaptively. While studied in first-order optimization~\cite{1DBLP:journals/corr/abs-2201-11803,13DBLP:journals/pvldb/YuanWDTKJ22},
research on second-order methods is limited.

\noindent
(3).
\textit{Training staleness:} Some model regions may remain
outdated due to asynchronous updates, slowing overall convergence.

\noindent
(4).
\textit{Data heterogeneity:} Variations in local data distributions
across workers introduce inconsistencies in gradient and Hessian computations,
affecting model accuracy.

\noindent
\textbf{Our Contributions.} We propose a novel algorithm, Distributed Adaptive
Newton Learning (\texttt{DANL}), that tackles the challenges of distributed
learning with heterogeneous data and workers. \texttt{DANL} uses a simple
Hessian initialization and adapts the training regions to the workers'
preferences. In each optimization round, the server broadcasts the global
model to all workers, who select their own masks and train diverse sub-models.
The server then collects and aggregates the local updates from the workers
for each region.

\texttt{DANL} avoids the costly computation and communication of the Hessian
matrix by initializing it once and reusing it throughout the optimization.
\texttt{DANL} also handles the issues of sub-model diversity, training
staleness, and data heterogeneity by introducing a server aggregation mechanism
that leverages the latest updates from different regions of the global
model in each worker and uses them as approximations for the current region
update. This approach effectively addresses the challenges and reduces
the server memory usage.
\texttt{DANL} can be widely applied in real-world scenarios
such as edge computing, federated learning, and autonomous systems, where
distributed optimization is essential. Its ability to handle large-scale,
heterogeneous environments makes it highly suitable for practical, resource-constrained
applications.

In summary, we present \texttt{DANL}, a new distributed stochastic optimization
algorithm based on Newton's method. We provide a rigorous convergence analysis
of \texttt{DANL} under standard stochastic optimization assumptions. The
analysis shows that \texttt{DANL} has a linear convergence rate and is
less sensitive to the condition number and parameter tuning than first-order
methods.
Experimental results verify that \texttt{DANL} delivers
rapid convergence in complex distributed environments.

\section{Related Work}
\label{sec2}

\textit{First-order methods.} Several first-order methods have been suggested
to solve distributed stochastic optimization problems, including distributed
SGD~\cite{25Pu_Olshevsky_Paschalidis_2021}, variance reduction SGD~\cite{26DBLP:conf/nips/ReddiHSPS15},
and accelerated SGD~\cite{27DBLP:conf/allerton/ShamirS14}. Pu et
al. propose DSGD, achieving optimal network-independent convergence rates
compared to centralized SGD~\cite{25Pu_Olshevsky_Paschalidis_2021}. Reddi
et al. investigate variance reduction algorithms in asynchronous
parallel and distributed settings~\cite{26DBLP:conf/nips/ReddiHSPS15}.
Shamir et al. comprehensively study distributed stochastic optimization,
showing the effectiveness of accelerated mini-batched SGD~\cite{27DBLP:conf/allerton/ShamirS14}.
While these methods reduce computation costs, they may increase communication
costs due to more communication rounds. Moreover, they have drawbacks regarding
the use of first-order information, possibly neglecting curvature details
in the objective function. Parameter tuning and dependence on the data's
condition number can be demanding.

\noindent
\textit{Second-order methods.} Newton's method, using the Hessian information,
has shown robustness in optimization problems, including federated learning
with Newton-type methods~\cite{23DBLP:conf/aistats/QianISR22,2DBLP:conf/icml/SafaryanIQR22},
communication-efficient Newton-type methods~\cite{20DBLP:conf/nips/BullinsPSSW21,18DBLP:journals/csysl/QiuZOL23},
and quasi-Newton methods~\cite{3DBLP:journals/tnn/ZhangHDH22,9DBLP:journals/tsp/ZhangLSL23}.
Qian et al. propose Basis Learn (BL) for federated learning, lowering
communication costs through a change of basis in the matrix space~\cite{23DBLP:conf/aistats/QianISR22}.
Safaryan et al. introduce Federated Newton Learn (FedNL) methods,
exploiting approximate Hessian information for better solutions~\cite{2DBLP:conf/icml/SafaryanIQR22}.
Bullins et al. present FEDSN, a communication-efficient stochastic
Newton algorithm~\cite{20DBLP:conf/nips/BullinsPSSW21}. Qiu et al.
develop St-SoPro, a stochastic second-order proximal method with decentralized
second-order approximation~\cite{18DBLP:journals/csysl/QiuZOL23}. Zhang
et al. devise SpiderSQN, a faster quasi-Newton method that integrates
variance reduction techniques~\cite{3DBLP:journals/tnn/ZhangHDH22}, and
propose stochastic quasi-Newton methods with linear convergence guarantees~\cite{9DBLP:journals/tsp/ZhangLSL23}.

However, there is scarce work addressing heterogeneous methods for second-order
optimization, especially considering resource constraints and sub-model
diversity. In this study, we explore distributed learning problems regarding
resource adaptation in stochastic Newton's method optimization.

\section{Model}
\label{sec3}

\subsection{Distributed learning with Newton's method}
\label{sec3.1}

We consider a distributed learning setup where a central server coordinates
with $N$ workers to solve a stochastic optimization problem. The problem
is to find an optimal model parameter vector $\omega ^{*}$ in the
$d$-dimensional parameter space $\mathbb{R}^{d}$ that minimizes the global
population risk $\hat{f}$, which is the expected average loss over all
workers and data samples. Formally, we have
%
\begin{equation}
\label{eq-goal}
\omega ^{*}=\argmin _{\omega \in \mathbb{R}^{d}}\left \{\hat{f}(
\omega ) \triangleq \frac{1}{N}\sum _{i=1}^{N} \mathbb{E}_{\zeta _{i}
\sim \mathcal{D}_{i}}[F_{i}(\omega ;\zeta _{i})]\right \},
\end{equation}
where $\mathcal{D}_{i}$ denotes the local data distribution of worker
$i$, $\zeta _{i}$ denotes a random data sample drawn from
$\mathcal{D}_{i}$, and $F_{i}(\omega ;\zeta _{i})$ denotes the per-sample
loss of worker $i$ caused by the model parameter $\omega $. For simplicity, we define $f_i(\omega)\triangleq\mathbb{E}_{\zeta _{i}
\sim \mathcal{D}_{i}}[F_{i}(\omega ;\zeta _{i})]$ so that $\hat{f}(\omega)=\frac{1}{N}\sum_{i=1}^N f_i(\omega)$.

The learning process is divided into synchronized communication rounds
of $t=1, 2, \cdots $, which produce a sequence of model parameter iterates
$\omega ^{1}, \omega ^{2}, \cdots , \omega ^{T}$. In each
round $t$, each worker $i\in[N]$ draws a random data sample
$\zeta _{i}^{t}\sim\mathcal{D}_{i}$. We define
%
\begin{equation}
\label{eq-eauploss}
\hat F(\omega ^{t}; \zeta _{1}^{t}, \zeta _{2}^{t}, \cdots , \zeta _{N}^{t})
\triangleq \frac{1}{N}\sum _{i=1}^{N} F_{i}(\omega ^{t}; \zeta _{i}^{t}),
\end{equation}
which represents the instant global loss incurred by the data samples generated
in round $t$ under model $\omega ^{t}$. The gradient
$\nabla \hat F^{t}$ and the Hessian $\mathbf{\Pi}^{t}$ of the loss function
at the current model parameter $\omega ^{t}$ are defined as

\[
\begin{aligned}
\nabla \hat F^{t} &\triangleq \nabla _{\omega}F(\omega ^{t}; \zeta _{1}^{t},
\zeta _{2}^{t}, \cdots , \zeta _{N}^{t}) \quad
\text{and} \quad
\mathbf{\Pi}^{t} &\triangleq \nabla ^{2}_{\omega} F(\omega ^{t}; \zeta _{1}^{t},
\zeta _{2}^{t}, \cdots , \zeta _{N}^{t}).
\end{aligned}
\]

We apply Newton's method to solve the problem
stated in \myref{eq-goal} by adapting the stochastic gradient descent
(SGD) framework. The update rule is given by
$ \omega ^{t+1} = \omega ^{t} - [\mathbf{\Pi}^{t}]^{-1} \nabla \hat F^{t}$,
where $[\mathbf{\Pi}^{t}]^{-1}$ denotes the inverse of
$\mathbf{\Pi}^{t}$.

\subsection{Resource-adaptive learning via pruning}
\label{sec3.2}

Online model pruning is a technique that reduces the model size and improves
the efficiency of distributed learning, as proposed in~\cite{1DBLP:journals/corr/abs-2201-11803}.
For $t$-th iteration, given the local loss function $F_{i}(\cdot;\cdot)$ at each worker $i$ and a general pruning policy
$\mathcal{M}$ that generates time-varying and client-specific pruning
masks $\tau _{i}^{t}\in \{0, 1\}^{d}$ for each worker $i$, the pruned local model at worker $i$  is defined as 
$\omega _{i}^{t}\triangleq \omega ^{t}\odot \tau _{i}^{t}$, where
$\odot $ is element-wise multiplication. 

To enable learning under pruned sub-models, each global model $\omega^t$ is partitioned into the same $Q$ disjoint regions, collectively denoted by the set $\mathcal{A}$. Each region contains a distinct number of parameters, allowing for flexible pruning strategies. Workers apply adaptive online masks to prune the global model into heterogeneous sub-models, with each sub-model consisting of various regions tailored to their local resources. Workers train only on the parameters contained within their respective pruned sub-models. Let $\mathcal{A}^t$ denote the set of regions selected for training in round $t$. For each region $q \in \mathcal{A}^t$, let $\mathcal{C}^{t,q}$ represent the set of workers assigned to train on region $q$ during round $t$. We define the minimum worker coverage number for the whole learning process as $\psi^* \triangleq \min_{t,q \in \mathcal{A}^t} |\mathcal{C}^{t,q}|$.

To quantify the delay between consecutive training rounds for each region, we define $\gamma_i^{t,q}$ as the maximum number of rounds during which region $q$ is not trained by worker $i$ up to round $t$. Additionally, we define $\gamma^t \triangleq \max_{q \in \mathcal{A}, i \in [N]} \gamma_i^{t,q}$ as the maximum delay across all regions and workers up to round $t$.

To evaluate the loss incurred by the heterogeneous pruned sub-models over random samples, we extend the notation from \myref{eq-eauploss} as follows:

\begin{equation*}
\hat{F}(\omega _{1}^{t}, \omega _{2}^{t}, \cdots , \omega _{N}^{t}; \zeta _{1}^{t},
\zeta _{2}^{t}, \cdots , \zeta _{N}^{t})\triangleq \frac{1}{N}\sum _{i=1}^{N}
F_{i}(\omega _{i}^{t}; \zeta _{i}^{t}).
\end{equation*}

For simplicity, when the context is clear, we use $F(\omega)$ to represent $\hat{F}(\omega_1^t, \omega_2^t, \dots, \omega_N^t; \zeta_1^t, \zeta_2^t, \dots, \zeta_N^t)$.

\subsection{Other notations}
\label{sec3.3}

We use $\Vert \cdot \Vert $ to denote the $\ell _{2}$ norm for vectors
or the spectral norm for matrices, and $\Vert \cdot \Vert _{F}$ to denote
the Frobenius norm for matrices.

\section{Algorithm}
\label{sec4}

We present our algorithm, \textit{Distributed Adaptive Newton Learning} (\texttt{DANL}), which utilizes the Hessian matrix $\mathbf{\Pi}$ of the initial model as the sole source of second-order information throughout the entire learning process. \texttt{DANL} operates in two distinct phases:

\noindent
\textbf{Phase I: Initialization.} Each worker $i$ computes the local Hessian $\nabla^2 F_i(\omega^0, \zeta_i^0)$ using the initial global parameter $\omega^0$ and sends it to the server. The server aggregates the local Hessians as
$\mathbf{\Pi} = \frac{1}{N}\sum _{i=1}^{N}\nabla ^{2} F_{i}(\omega ^{0},
\zeta _{i}^{0})$. 
This matrix is then projected to obtain $[\mathbf{\Pi}]_{\mu}$, which is used to update the global parameter at each step. The projection is defined as follows:

\begin{definition}[Projection~\cite{2DBLP:conf/icml/SafaryanIQR22}]%
\label{definiton:projection}
The projection of a symmetric matrix $\mathbf{A}$ onto the cone of positive semi-definite matrices
$\{\mathbf{H} \in \mathbb{R}^{d \times d}: \mathbf{H}^{\top }=
\mathbf{H}, \mu \mathbf{I} \preceq \mathbf{H} \}$ is given by
\begin{equation*}
[\mathbf{A}]_{\mu}:=[\mathbf{A}-\mu \mathbf{I}]_{0}+\mu \mathbf{I},
\end{equation*}
where, for any matrix $\mathbf{M}$, $[
\mathbf{M}]_{0}:=\sum \limits _{i=1}^{d} \max \left \{\lambda _{i}, 0
\right \} u_{i} u_{i}^{\top}$ and $\sum_{i} \lambda_i u_i u_i^\top$ represents the eigenvalue decomposition of $\mathbf{M}$..
\end{definition}

\noindent
\textbf{Phase II: Resource-Adaptive Learning.} For each round $t \in {1, 2, \dots, T}$, all workers receive the latest global parameter $\omega^t$ from the server. Each worker then generates a local adaptive mask $\tau_i^t$, which is used to compute a \textit{pruned} local gradient:
$\nabla F_{i}^{t} \triangleq \nabla F_{i}(\omega ^{t}\odot \tau _{i}^{t},
\zeta _{i}^{t})\odot \tau _{i}^{t}$. The server aggregates the pruned local gradients, denoted as $\nabla F^{t}$. Since the global model is divided into $Q$ disjoint regions (as defined by $\mathcal{A}$), the gradient $\nabla F^{t}$ is correspondingly split into $Q$ fragments. Each fragment, denoted as $\nabla F^{t,q}$, represents the gradient for region $q$. As some regions may not be trained by any worker during round $t$, the server maintains $\Theta_i^{t,q}$, the most recent local gradient fragment for region $q$ uploaded by worker $i$. The global gradient fragment $\nabla F^{t,q}$ is then obtained by aggregating all $\Theta_i^{t,q}$ values for $i \in [N]$, and $\nabla F^t$ is reconstructed by combining all $\nabla F^{t,q}$ values for $q \in [Q]$. Finally, $\nabla F^t$ is used to update $\omega^t$ via a preconditioned gradient descent step.

The detailed steps of \texttt{DANL} are provided in Algorithm~\ref{algo1}.

%
\begin{algorithm2e}[!htbp]
    \label{algo1}
    \caption{\texttt{DANL}: Distributed Adaptive Newton Learning.}
    \LinesNumbered
    \KwIn{
        Local datasets $\{\mathcal{D}_{i}\}_{i=1}^{N}$, pruning policy $\mathcal{M}$, initialized model $\omega ^{0}$.
    }
    \KwOut{
    $\omega ^{t}$.
    }
    \tcc{Phase I: Initialization}
    \textcolor{blue!80}{\textbf{Server}} broadcasts $\omega ^{0}$ to all workers\;
    \For{\textbf{\textup{every}} \textcolor{orange}{\textbf{worker} $\mathbf{i}$} \textbf{\textup{in parallel}}}{
        Compute local gradient $\nabla F_{i}(\omega ^{0}, \zeta _{i}^{0})$ and local Hessian $\nabla ^{2} F_{i}(\omega ^{0}, \zeta _{i}^{0})$ \;
        Send $\nabla F_{i}(\omega ^{0}, \zeta _{i}^{0})$ and $\nabla ^{2} F_{i}(\omega ^{0}, \zeta _{i}^{0})$ to the server \;
    }
    \textcolor{blue!80}{\textbf{Server}} aggregates the local Hessian matrices: $\mathbf{\Pi} = \frac{1}{N}\sum \limits _{i=1}^{N}\nabla ^{2} F_{i}(\omega ^{0}, \zeta _{i}^{0})$ \;
    \textcolor{blue!80}{\textbf{Server}} initializes $\Theta _{i}^{0,q}$ for each worker $i$ and region $q$: $\Theta _{i}^{0,q}= \nabla F_{i}^{q}(\omega ^{0}, \zeta _{i}^{0}) $ \;
    \textcolor{blue!80}{\textbf{Server}} updates the global model: $\omega ^{1} = \omega ^{0} - [\mathbf{\Pi}]_{\mu}^{-1} \left (\nicefrac{\sum \limits _{i=1}^{N}\nabla F_{i}(\omega ^{0}, \zeta _{i}^{0})}{N}\right )$ \;
    \textcolor{blue!80}{\textbf{Server}} broadcasts $\omega ^{1}$ to all workers \;
    \tcc{Phase II: Resource-Adaptive Learning}
    \For{\textbf{\textup{round}} $t=1, \cdots , T$}{
        \For{\textbf{\textup{every}} \textcolor{orange}{\textbf{worker} $\mathbf{i}$} \textbf{\textup{in parallel}}}{
            Generate the mask $\tau _{i}^{t} = \mathcal{M}(\omega ^{t},i)$ \;
            Prune the model: $\omega _{i}^{t} = \omega ^{t} \odot \tau _{i}^{t}$ \;
            Compute the gradient after pruning: $\nabla F_{i}^{t}= \nabla F_{i}(\omega _{i}^{t}, \zeta _{i}^{t})\odot \tau _{i}^{t}$ \;
            Send $\nabla F_{i}^{t}$ to the server \;
        }
        \For{\textbf{\textup{every region}} $q=1, \cdots , Q$}{
             \textcolor{blue!80}{\textbf{Server}} finds $\mathcal{C}^{t,q}=\{i:\tau _{i}^{t,q}=\mathbf{1}\}$\;
                  \For{$i=1$ to $N$}{
                  $\Theta _{i}^{t,q}=\left \{
                    \begin{aligned}
                        \nabla F_{i}^{t,q}& \quad if \quad i \in \mathcal{C}^{t,q} \\
                        \Theta _{i}^{t-1,q}& \quad if \quad i \notin  \mathcal{C}^{t,q}
\end{aligned}
                    \right .$ \;
                    }
            {\textcolor{blue!80}{\textbf{Server}} updates $\nabla F^{t,q}=\frac{1}{N}\sum \limits _{i=1}^{N}\Theta _{i}^{t,q}$ \;}
        }
        \textcolor{blue!80}{\textbf{Server}} updates the global parameter: $\omega ^{t+1} = \omega ^{t} - [\mathbf{\Pi}]_{\mu}^{-1} \nabla F^{t}$ \;
        \textcolor{blue!80}{\textbf{Server}} broadcasts $\omega ^{t+1}$ to all workers \;
    }
\end{algorithm2e}
\section{Convergence results}
\label{Theoretical_results_VTEX1}%

In this section, we analyze the performance of Algorithm~\ref{algo1}. Algorithm~\ref{algo1}
shows the numerical linear convergence rate. We first introduce some key
concepts in optimization.

\begin{definition}[Lipschitz]%
\label{Lipschitz}
A function $\ell :\mathbb{R}^{d}\to \mathbb{R}^{d}$ is $\iota $-Lipschitz
if for $\forall \omega _{1}, \omega _{2}\in \mathbb{R}^{d}$,
\begin{equation*}
\lVert \ell (\omega _{1})-\ell (\omega _{2})\lVert \le \iota \lVert
\omega _{1}-\omega _{2}\rVert .
\end{equation*}
Then, the function $\ell $ has $\iota _{g}$ Lipschitz continuous gradient
and $\iota _{h}$ Lipschitz continuous Hessian if for
$\forall \omega _{1}, \omega _{2}\in \mathbb{R}^{d}$,
\[
\begin{aligned}
\lVert \nabla \ell (\omega _{1}) - \nabla \ell (\omega _{2}) \rVert 
&\le \iota _{g} \lVert \omega _{1} - \omega _{2} \rVert, \quad 
\lVert \nabla ^{2} \ell (\omega _{1}) - \nabla ^{2} \ell (\omega _{2}) \rVert 
\le \iota _{h} \lVert \omega _{1} - \omega _{2} \rVert.
\end{aligned}
\]

\end{definition}

\begin{definition}[Bounded variance]
\label{defn3}
Define a function $L(\omega ,\zeta )$ and its unbiased expectation function
$\ell (\omega )$. $\ell (\omega )$ is an unbiased estimator of
$L(\omega ,\zeta )$ with a bounded variance if for
$\forall \omega \in \mathbb{R}^{d}$,
\begin{equation*}
\mathbb{E}_{\zeta \sim \mathcal{D}}\Vert L(\omega ,\zeta ) - \ell (
\omega )\Vert ^{2} \leq \kappa ^{2}.
\end{equation*}
\end{definition}

\begin{definition}[Strong convexity]
\label{defn4}
A differentiable function $\ell :\mathbb{R}^{d}\to \mathbb{R}^{d}$ is
$\mu $-strongly convex if for
$\forall \omega _{1}, \omega _{2}\in \mathbb{R}^{d}$,
\begin{equation*}
\ell (\omega _{1}) \geq \ell (\omega _{2}) + \left \langle \nabla
\ell (\omega _{2}), \omega _{1}-\omega _{2} \right \rangle +
\frac{\mu}{2}\Vert \omega _{1} - \omega _{2} \Vert .
\end{equation*}
\end{definition}
For twice differentiable functions, strong convexity is equivalent to the
requirement that the minimum eigenvalue of Hessian is above zero, i.e.,
$\lambda _{\min}\left (\nabla ^{2}\ell (\omega )\right )$
$\geq \mu $.

The following necessary assumptions are made for the analysis.

\begin{assumption}%
\label{A-strongly_convex_VTEX1}
The global population risk function $\hat{f}(\omega )$ is $\mu $-strongly convex.
\end{assumption}

\begin{assumption}%
\label{A-Lipschitz}
For $\forall i\in [N]$, $f_{i}(\omega )$ is $\iota $-Lipschitz
for any $\omega \in \mathbb{R}^{d}$ and twice continuously differentiable
in respect of $\omega \in \mathbb{R}^{d}$. Each function
$f_{i}(\omega )$ has a $\iota _{g}$-Lipschitz gradient and a
$\iota _{h}$-Lipschitz Hessian for any $\omega \in \mathbb{R}^{d}$.
\end{assumption}

\begin{assumption}%
\label{A-variance}
There exists a constant $\Lambda \geq 0$ and a constant
$\kappa \geq 0$ such that: (\textrm{i}) for $\forall i\in [N]$ and $\forall \omega \in \mathcal{R}^{d}$, it holds that
$\mathbb{E}_{\zeta \sim \mathcal{D}_{i}}\Vert \nabla F_{i}(\omega ,
\zeta ) - \nabla f_{i}(\omega )\Vert ^{2} \leq \Lambda ^{2}$; (\textrm{ii})
for $\forall i\in [N]$ and the initial $\omega ^{0} \in \mathcal{R}^{d}$,
$\mathbb{E}_{\zeta \sim \mathcal{D}_{i}}\Vert \nabla ^{2} F_{i}(
\omega ^{0},\zeta ) - \nabla ^{2} f_{i}(\omega ^{0})\Vert _{F}^{2}
\leq \kappa ^{2}$.
\end{assumption}

\begin{assumption}%
\label{A-prun}
For $\forall i\in [N]$ and $\forall t\in [T]$, there exists a constant
$\delta \geq 0$ such that
$\mathbb{E}\Vert \omega ^{t}-\omega _{i}^{t} \Vert ^{2} \leq \delta ^{2}$.
\end{assumption}

The assumptions above are fairly standard and have been widely adopted in prior works, e.g.,~\cite{1DBLP:journals/corr/abs-2201-11803,2DBLP:conf/icml/SafaryanIQR22,3DBLP:journals/tnn/ZhangHDH22}. Next, we introduce the following auxiliary notations to facilitate the analysis:

\begin{equation}
	\nabla f_{i}^{t}\triangleq \mathbb{E}_{\zeta_i^t\sim\mathcal{D}_i}[\nabla F_{i}^{t}]= \mathbb{E}_{\zeta_i^t\sim\mathcal{D}_i}[\nabla F_{i}(\omega _{i}^{t}, \zeta _{i}^{t})\odot \tau _{i}^{t}].
\end{equation}

\begin{equation}
	\forall t\ge0, \quad \theta _{i}^{t,q}\triangleq\left \{
                    \begin{aligned}
                        \nabla f_{i}^{t,q}& \quad \text{if} \quad i \in \mathcal{C}^{t,q}, \\
                        \theta _{i}^{t-1,q}& \quad \text{if} \quad i \notin  \mathcal{C}^{t,q} .
\end{aligned}
                    \right .
\end{equation}

\begin{equation}
	\nabla f^{t,q}\triangleq\frac{1}{N}\sum \limits _{i=1}^{N}\theta _{i}^{t,q}.
\end{equation}

We establish the linear convergence of \texttt{DANL} in Theorem~\ref{theorem-rate}. To prove Theorem~\ref{theorem-rate}, we introduce four preliminary results, presented as Lemmas~\ref{le-projection} through~\ref{le-pruning_error_VTEX1}. The detailed proofs for all results are provided in Appendix~\ref{Theoretical_Analysis_VTEX1}.

\begin{lemma}%
\label{le-projection}
For the projected Hessian
$[\mathbf{\Pi}]_{\mu }$ computed according to Definition~\ref{definiton:projection}, we have 
\begin{equation*}
\left \|\mathbf{\Pi}_{\mu}-\mathbf{\Pi}^{*}\right \|_{\mathrm{F}}
\leq \left \|\mathbf{\Pi}-\mathbf{\Pi}^{*}\right \|_{\mathrm{F}},\quad\text{where}\quad\mathbf{\Pi}^{*} \triangleq \nabla ^{2} \hat{f}\left (\omega ^{*}\right ).
\end{equation*}
\end{lemma}

\begin{proof}
The proof follows from Lemma $C.2$ of~\cite{2DBLP:conf/icml/SafaryanIQR22}.
\end{proof}
Lemma~\ref{le-projection} establishes that the proximity of the projected Hessian matrix $[\mathbf{\Pi}]_{\mu}$ to the optimal Hessian matrix $\mathbf{\Pi}^*$ is no greater than the proximity of the pre-projected Hessian matrix $\mathbf{\Pi}$ to $\mathbf{\Pi}^*$.

\begin{lemma}%
\label{le-stochastic_hessian_error_VTEX1}
Under Assumption~\ref{A-variance}(ii), we have
\begin{equation*}
\mathbb{E}\left \|\mathbf{\Pi}-\nabla ^{2} \hat{f}\left (\omega ^{*}
\right )\right \|^{2}\leq 2\mathbb{E}\left \|\nabla ^{2} \hat{f}
\left (\omega ^{0}\right )-\nabla ^{2} \hat{f}\left (\omega ^{*}
\right )\right \|_{F}^{2}+2 \kappa ^{2}.
\end{equation*}
\end{lemma}

Lemma~\ref{le-stochastic_hessian_error_VTEX1} states that the error in using $\mathbf{\Pi}$ to approximate the Hessian of the population risk $\hat{f}$ at $\omega^*$ can be decomposed into two terms: the bias arising from the discrepancy between $\omega_0$ and $\omega^*$, and the variance introduced by data sampling when estimating the Hessian of the population risk $\hat{f}$ at $\omega_0$.

\begin{lemma}%
\label{le-stochastic_gradient_error_VTEX1}
Under Assumption~\ref{A-variance}(i), we have
\begin{equation*}
\begin{aligned}
\mathbb{E}\left \|\nabla ^{2} \hat{f}\left (\omega ^{*}\right )
\left (\omega ^{t}-\omega ^{*}\right )-\nabla F^{t}+\nabla \hat{f}
\left (\omega ^{*}\right )\right \|^{2}
\leq 2\mathbb{E}\left \|\nabla ^{2} \hat{f}\left (\omega ^{*}\right )
\left (\omega ^{t}-\omega ^{*}\right )-\nabla f^{t}+\nabla \hat{f}
\left (\omega ^{*}\right )\right \|^{2}+2\left(\frac{N}{\psi ^{*}}+1\right) \Lambda ^{2}.
\end{aligned}
\end{equation*}
\end{lemma}

Lemma~\ref{le-stochastic_gradient_error_VTEX1} essentially indicates that the variance in estimating $\nabla f^t$ using $\nabla F^t$, i.e., $\mathbb{E}[\|\nabla f^t-\nabla F^t\|^2]$, is bounded above by $2\left(\frac{N}{\psi ^{*}}+1\right) \Lambda ^{2}$. Subsequently, Lemma~\ref{le-pruning_error_VTEX1} demonstrates that the first term on the right-hand side of Lemma~\ref{le-stochastic_gradient_error_VTEX1} can also be bounded from above.

\begin{lemma}%
\label{le-pruning_error_VTEX1}
Under Assumption~\ref{A-Lipschitz} and~\ref{A-prun}, we have
\begin{equation*}
\begin{aligned}
\mathbb{E}\left \|\nabla f^{t}-\nabla \hat{f}\left (\omega ^{*}
\right )-\nabla ^{2} \hat{f}\left (\omega ^{*}\right )\left (\omega ^{t}-
\omega ^{*}\right )\right \|^{2}
\leq \frac{N}{\psi ^{*}}\left (2\iota _{g}^{2}\delta ^{2}\right )+
(\gamma^{t})^{2} \frac{8(L^{2}+\Lambda ^{2})\iota _{g}^{2}}{\mu ^{2}}+4\iota _{g}^{2}
\delta ^{2}+ \frac{\iota _{h}^{2}}{2}\mathbb{E}\left \|
\omega ^{t}-\omega ^{*}\right \|^{4}.
\end{aligned}
\end{equation*}
\end{lemma}

To present Theorem~\ref{theorem-rate}, we introduce the following notations:
$a= \frac{\iota _{h}^{2}}{2}$,
$b=\frac{\mu ^{2}}{16}-\kappa ^{2}$,
$c=\frac{N}{\psi ^{*}}\left (2\iota _{g}^{2}\delta ^{2}+\Lambda ^{2}
\right )+(\gamma^{t})^{2} \frac{8(L^{2}+\Lambda ^{2})\iota _{g}^{2}}{\mu ^{2}}+4
\iota _{g}^{2}\delta ^{2}+ \Lambda ^{2}$ and $\rho =b^{2}-4ac$.

\begin{theorem}%
\label{theorem-rate}
Let Assumption~\ref{A-strongly_convex_VTEX1},~\ref{A-Lipschitz},~\ref{A-variance}
and~\ref{A-prun} hold. Suppose (i) $\mu\ge 4\kappa$, (ii) $\rho\ge 0$, (iii) $\frac{b-\sqrt{\rho}}{2a}\le\frac{\mu}{4\iota_h}$, and (iv) $\frac{b-\sqrt{\rho}}{2a}\le\|\omega^0-\omega^*\|\le\min\left\{\frac{\mu}{4\iota_h}, \frac{b+\sqrt{\rho}}{2a}\right\}$. Then \texttt{DANL} (Algorithm~\ref{algo1}) converges linearly with the rate
\begin{equation*}
\mathbb{E}\|\omega ^{t}-\omega ^{*}\|^{2}\leq \frac{1}{2^t}
\mathbb{E}\left \|\omega ^{0}-\omega ^{*}\right \|^{2}.
\end{equation*}
\end{theorem}

From Theorem~\ref{theorem-rate}, it can be observed that, with an appropriately chosen initial model, the local linear convergence rate of the model iterates depends solely on a universal numerical constant $\frac{1}{2}$, rather than any problem-specific constants. This implies that we obtain a the rate that is independent of the condition number.

\section{Experiments}
\label{sec6}

In this section, we conduct extensive experiments to study the performance
of our proposed algorithm DANL. We consider the problem \myref{eq-goal} with local loss functions:
\begin{align*}
& F_{i}(\omega ) = -\frac{1}{m}\left (\sum _{j = 1}^{m}
\left (b_{i,j}\log p(a_{i,j}^{T}\omega ) + (1 - b_{i,j})\log \left (1 -
p\left (a_{i,j}^{T}\omega \right )\right )\right )\right ) +
\frac{\lambda}{2m}||\omega ||^{2},
\end{align*}
where $p(z) = \frac{1}{1 + e^{-z}}$.
$\{a_{i,j}, b_{i,j}\}_{j\in [m]}$ are data points at the $i$-th device
and $\lambda >0$ is a regularization parameter. The datasets were taken
from LibSVM library~\cite{chang2011libsvm}: a1a, a2a, a3a, and phishing.

\subsection{Parameter setting}
\label{sec6.1}

In the distributed optimization framework, we construct submodels with
parameters allocated in a non-deterministic manner. This approach implies
that not every parameter set is updated in each iteration. To achieve this
goal, we randomly select partial regions not to be trained periodically.
The global model, denoted as $\omega $, is divided into
four distinct segments:
$\omega = \{\omega ^{1}, \omega ^{2}, \omega ^{3}, \omega ^{4}\}$. For
instance, a client may choose to update only the segments
$\{\omega ^{1}, \omega ^{3}\}$ during a given iteration.

A single regularization parameter, $\lambda = 10^{-4}$, is utilized. The
figures illustrate the relationship between the optimality gap,
$F(\omega ^{k}) - F(\omega ^{*})$, and the progression of communication
rounds. Here, $F(\omega ^{*})$ represents the function
value after the $20$-th iteration using the standard Newton's method, serving
as our benchmark for optimality. The initial parameter set,
$\omega ^{0}$, is derived from the $10$-th iteration output
of FedAvg~\cite{DBLP:conf/aistats/McMahanMRHA17} algorithm.

We set $10$ workers in the training process. The experiments
were designed to evaluate the impact of three key parameters: the minimum
coverage number $\psi ^{*}$, the maximum number of continuously non-trained rounds
$\gamma^{t}$ and the minimum number of trained regions
$S^{*} \triangleq \min_{t} |\mathcal{A}_t| $, where $\mathcal{A}_t $ is the set of trained regions in round $t$, representing the minimum number of trained regions across all rounds.

\subsection{Impact of key factors}
\label{sec6.2}

\textit{Impact of minimum coverage number $\psi ^{*}$ and minimum number
of trained regions $S^{*}$}. Combined with our theoretical analysis, we
study two key factors impacting convergence: $\psi ^{*}$ and $S^{*}$. Through
varying $\psi ^{*}$ and $S^{*}$, we set three distinct scenarios: ($
\psi ^{*}=1, S^{*}=4$), ($\psi ^{*}=3, S^{*}=4$), ($\psi ^{*} = 10, S^{*}=1$),
as shown in Fig.~\ref{fig1}.

\begin{figure}[!htb]
\centering
\subfigure[a1a]{
\includegraphics[scale=0.45]{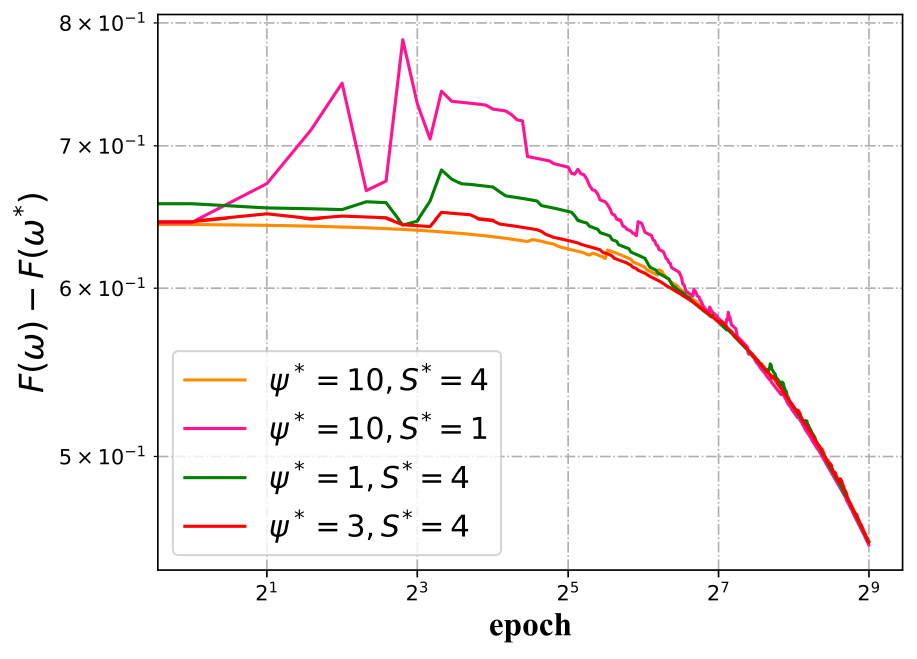} 
}
\quad
\subfigure[a2a]{
\includegraphics[scale=0.45]{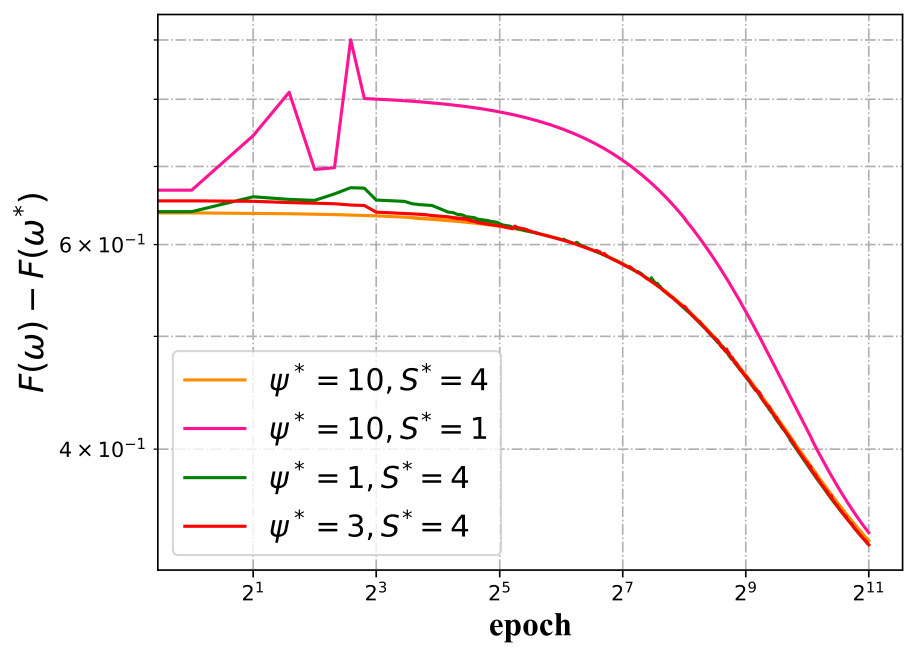} 
}
\quad
\subfigure[a3a]{
\includegraphics[scale=0.45]{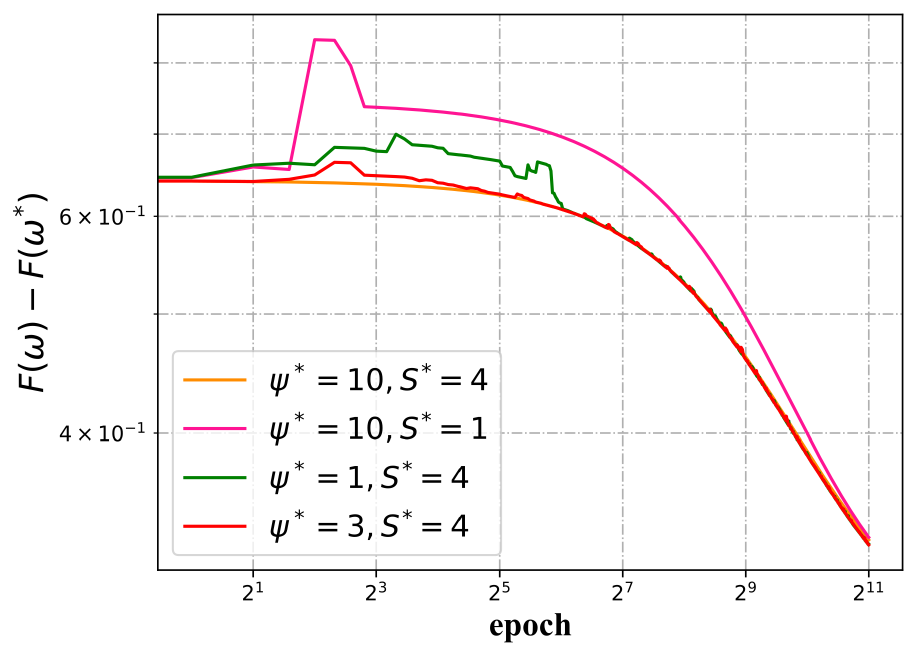} 
}
\quad
\subfigure[phishing]{
\includegraphics[scale=0.45]{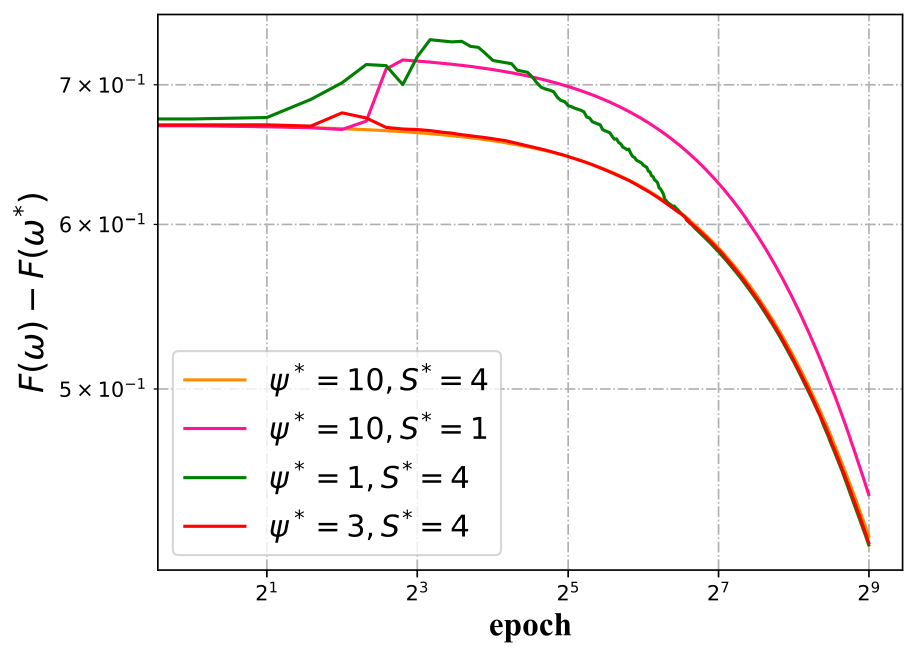}
}
\caption{The impact of $\psi^*$ and $S^{*}$ of DANL}
\label{fig1}
\end{figure}

When fixing $S^{*}=4$, we observe that $\psi ^{*}=3$ converges faster compared
to $\psi ^{*}=1$, which is consistent with our theoretical analysis. It
is worth noting that in the early stages of training, the loss function
initially experiences a period of increase before gradually decreasing.
This is primarily because initially only subregions can be trained, and
as training progresses sufficiently, all regions receive adequate training.

The increase in $\psi ^{*}$ does not necessarily imply faster convergence
speed. Consider an extreme example with the largest $\psi ^{*}$ (e.g. $\psi ^{*}=10$,
$S^{*}=1$). In this scenario, training is restricted to a
single subregion per round within a set of $10$ submodels. However, the
convergence speed of the loss function is very slow. This indicates that
when many subregions remain untrained, the model performance deteriorates.
Therefore, we can conclude that a combination of factors influences performance
and convergence rate.

\textit{Impact of the maximum number of continuously non-trained rounds
$\gamma^{t}$.} We further consider two key factors impacting convergence
in DANL: $\psi ^{*}$ and $\gamma^{t}$. Fixing other impacting factors
in DANL: we set $S^{*} = 4$. Considering $\psi ^{*}$ and
$\gamma^{t}$, we set four combinations: ($\psi ^{*}=1, \gamma^{t}=4$),
($\psi ^{*}=2, \gamma^{t}=4$), ($\psi ^{*}=1, \gamma^{t}=2$), ($
\psi ^{*}=2, \gamma^{t}=2$), as shown in Fig.~\ref{fig2}.
\begin{figure}[!htb]
\centering
\subfigure[a1a]{
\includegraphics[scale=0.45]{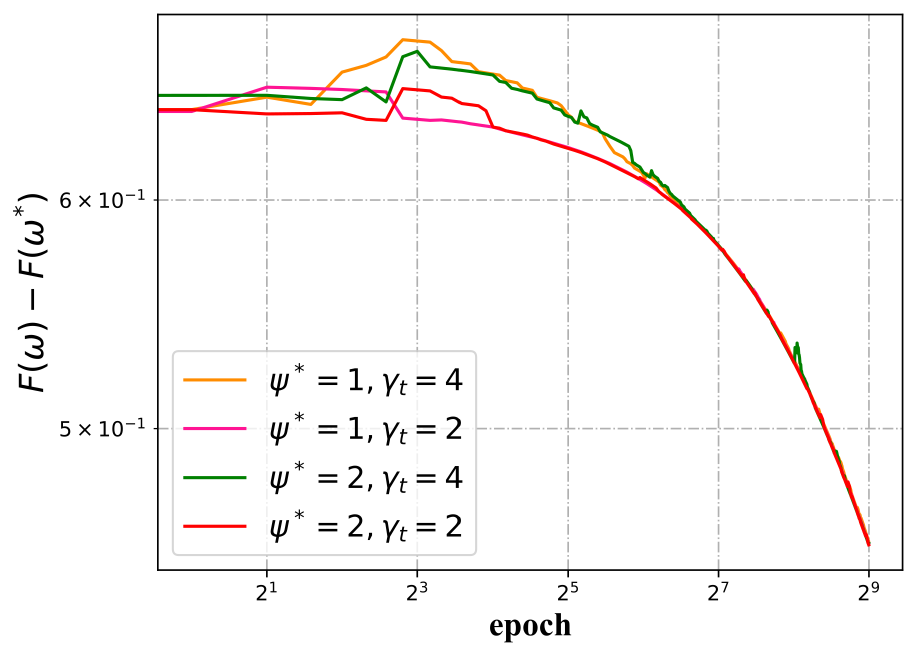} 
}
\quad
\subfigure[a2a]{
\includegraphics[scale=0.45]{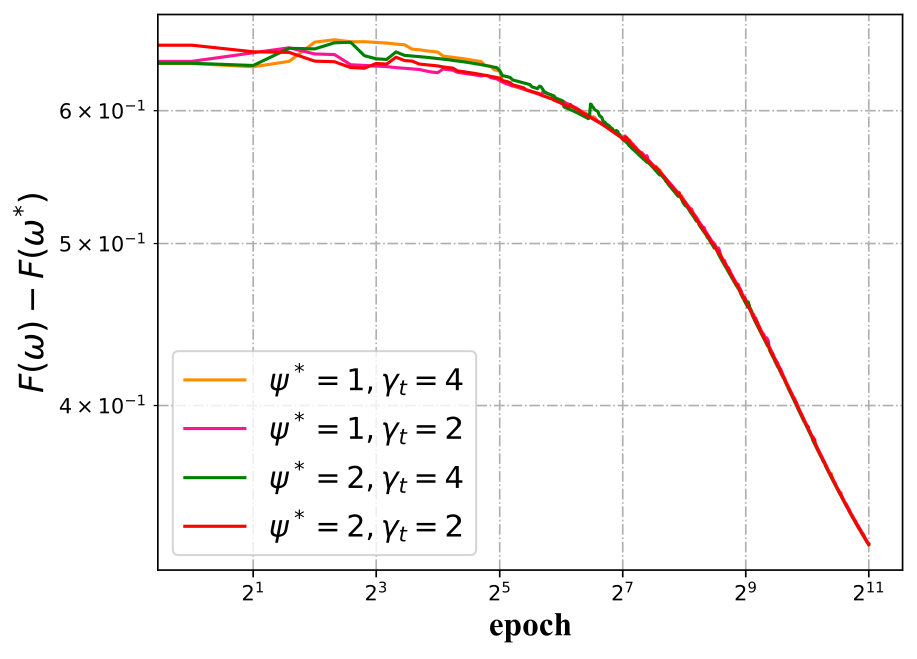}  
}
\quad
\subfigure[a3a]{
\includegraphics[scale=0.45]{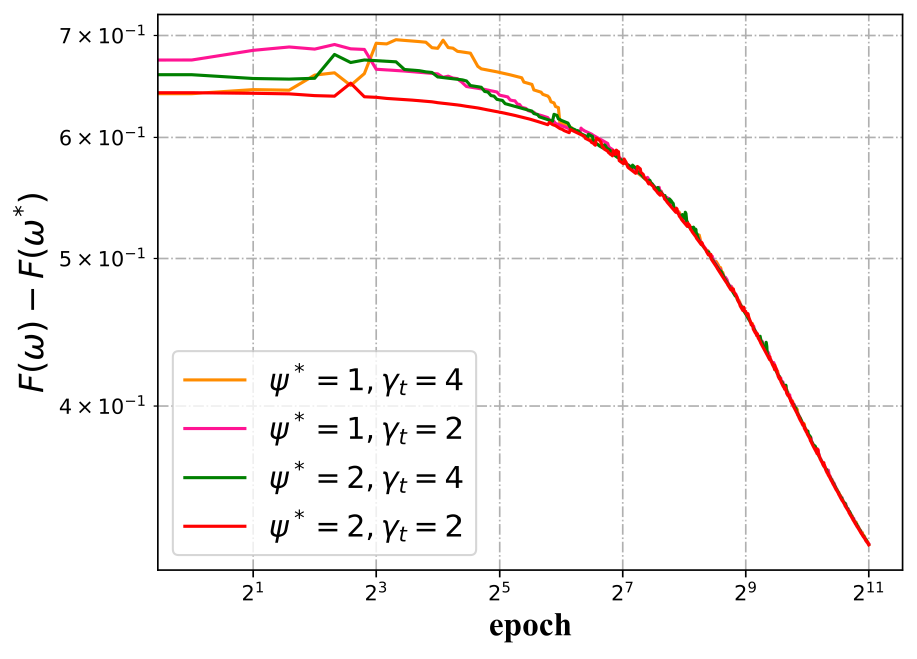}
}
\quad
\subfigure[phishing]{
\includegraphics[scale=0.45]{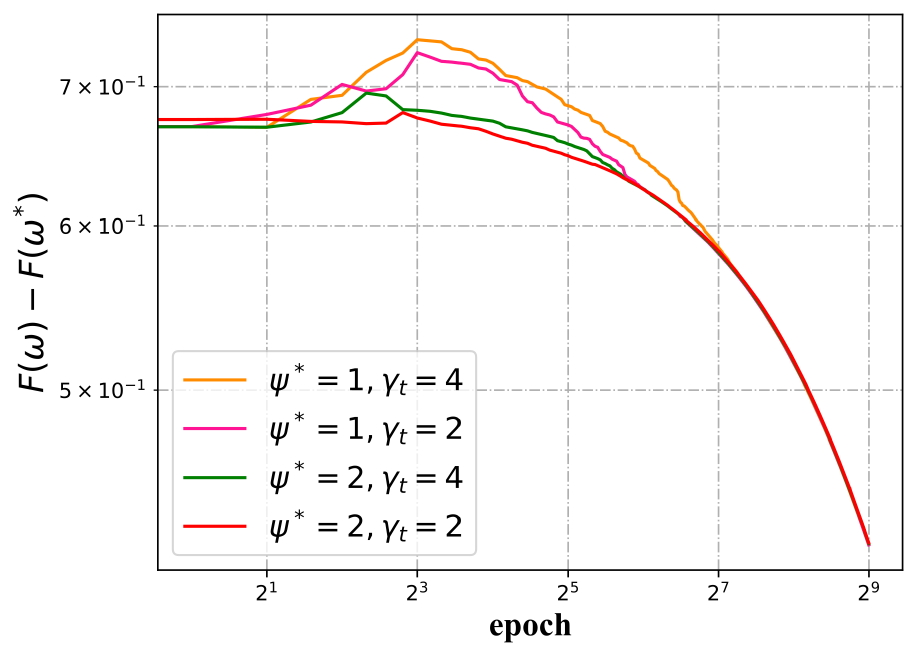}
}
\caption{The impact of $\psi^*$ and $\gamma_t$ of DANL}
\label{fig2}
\end{figure}

When $\psi ^{*}$ is fixed, $\gamma^{t}=2$ outperforms
$\gamma^{t}=4$ in terms of convergence speed. This is because, with
$\gamma^{t}=2$, each parameter receives more thorough training. Similarly,
when $\gamma^{t}$ is fixed, it is evident that with the increase of
$\psi ^{*}$, performance improves, further indicating that larger
$\psi ^{*}$ leads to more comprehensive training of subregions. Overall,
$\psi ^{*}$ plays a dominant role compared to $\gamma^{t}$.
$\gamma^{t}$ has a slight impact on convergence.

\newpage
\appendix

\section{Appendix}\label{Theoretical_Analysis_VTEX1}
\begin{proof}[Proof of Lemma~\ref{le-stochastic_hessian_error_VTEX1}]
\begin{align*}
&~~~~ \mathbb{E}\left \|\mathbf{\Pi}-\nabla ^{2} \hat{f}\left (
\omega ^{*}\right )\right \|^{2}
\\
&= \mathbb{E}\left \|\nabla ^{2}\hat{F}\left (\omega ^{0}\right )-
\nabla ^{2} \hat{f}\left (\omega ^{*}\right )\right \|^{2}
\\
&=\mathbb{E}\left \|\nabla ^{2} \hat{f}\left (\omega ^{0}\right )-
\nabla ^{2} \hat{f}\left (\omega ^{*}\right )+\nabla ^{2}\hat{F}
\left (\omega ^{0}\right )-\nabla ^{2} \hat{f}\left (\omega ^{0}
\right )\right \|^{2}
\\
&\le\mathbb{E}\left \|\nabla ^{2} \hat{f}\left (\omega ^{0}\right )-
\nabla ^{2} \hat{f}\left (\omega ^{*}\right )+\nabla ^{2}\hat{F}
\left (\omega ^{0}\right )-\nabla ^{2} \hat{f}\left (\omega ^{0}
\right )\right \|_{F}^{2}
\\
&\leq 2\mathbb{E}\left \|\nabla ^{2} \hat{f}\left (\omega ^{0}\right )-
\nabla ^{2} \hat{f}\left (\omega ^{*}\right )\right \|_{F}^{2}+2
\mathbb{E}\left \|\nabla ^{2} \hat{F}\left (\omega ^{0}\right )-
\nabla ^{2} \hat{f}\left (\omega ^{*}\right )\right \|_{F}^{2}
\\
&\leq 2\mathbb{E} \left \|\nabla ^{2} \hat{f}\left (\omega ^{0}
\right )-\nabla ^{2} \hat{f}\left (\omega ^{*}\right )\right \|_{F}^{2}+2
\mathbb{E}\left \|
\frac{\sum \limits _{i=1}^{N} \nabla ^{2} F_{i}\left (\omega ^{0}, \zeta _{i}^{0}\right )-\sum \limits _{i=1}^{N} \nabla ^{2} f_{i}\left (\omega ^{0}\right )}{N}
\right \|_{F}^{2}
\\
&\leq 2\mathbb{E}\left \|\nabla ^{2} \hat{f}\left (\omega ^{0}\right )-
\nabla ^{2} \hat{f}\left (\omega ^{*}\right )\right \|_{F}^{2}+
\frac{2}{N^{2}} \cdot N \sum \limits _{i=1}^{N} \mathbb{E}\left \|
\nabla ^{2} F_{i}\left (\omega ^{0}, \zeta _{i}^{0}\right )-\nabla ^{2}
f_{i}\left (\omega ^{0}\right )\right \|_{F}^{2}
\\
&\overset{\text{Assumption~\ref{A-variance}}(ii)}{\leq} 2\mathbb{E}
\left \|\nabla ^{2} \hat{f}\left (\omega ^{0}\right )-\nabla ^{2}
\hat{f}\left (\omega ^{*}\right )\right \|_{F}^{2}+2 \kappa ^{2}.\qedhere
\end{align*}
\end{proof}

\begin{proof}[Proof of Lemma~\ref{le-stochastic_gradient_error_VTEX1}]
{\small

\begin{align*}
&~~~~~ \mathbb{E}\left \|\nabla ^{2} \hat{f}\left (\omega ^{*}\right )
\left (\omega ^{t}-\omega ^{*}\right )-\nabla F^{t}+\nabla \hat{f}
\left (\omega ^{*}\right )\right \|^{2}
\\
&=\mathbb{E}\left \|\nabla ^{2} \hat{f}\left (\omega ^{*}\right )
\left (\omega ^{t}-\omega ^{*}\right )-\nabla f^{t}+\nabla \hat{f}
\left (\omega ^{*}\right )+\nabla f^{t}-\nabla F^{t}\right \|^{2}
\\
&\leq 2\mathbb{E}\left \|\nabla ^{2} \hat{f}\left (\omega ^{*}\right )
\left (\omega ^{t}-\omega ^{*}\right )-\nabla f^{t}+\nabla \hat{f}
\left (\omega ^{*}\right )\right \|^{2}+2 \mathbb{E}\|{\nabla f^{t}-
\nabla F^{t} \|^{2}}
\\
&\leq 2 \mathbb{E}\left \|\nabla ^{2}\hat{f}\left (\omega ^{*}\right )
\left (\omega ^{t}-\omega ^{*}\right )-\nabla f^{t}+\nabla \hat{f}
\left (\omega ^{*}\right )\right \|^{2}+2\sum \limits _{q\in
\mathcal{A}^{t}}\mathbb{E}\left \|
\frac{1}{\left |\mathcal{C}^{t,q}\right |}\sum \limits _{i \in
\mathcal{C}^{t,q} } (\nabla f_{i}^{t,q}-\nabla F_{i}^{t,q})\right \|^{2}
+2\sum \limits _{q\in \mathcal{A}-\mathcal{A}^{t}}\mathbb{E}
\left \|\frac{1}{N}\sum \limits _{i=1}^{N} ( \theta _{i}^{t,q}-
\Theta _{i}^{t,q})\right \|^{2}
\\
&\leq 2\mathbb{E}\left \|\nabla ^{2} \hat{f}\left (\omega ^{*}\right )
\left (\omega ^{t}-\omega ^{*}\right )-\nabla f^{t}+\nabla \hat{f}
\left (\omega ^{*}\right )\right \|^{2}+2 \sum \limits _{q\in
\mathcal{A}^{t}} \frac{1}{\left |\mathcal{C}^{t,q}\right |} \sum
\limits _{i \in \mathcal{C}^{t,q} } \mathbb{E}\left \|\nabla f_{i}^{t,q}-
\nabla F_{i}^{t,q}\right \|^{2}
+2\sum \limits _{q\in \mathcal{A}-\mathcal{A}^{t}}\frac{1}{N}
\sum \limits _{i =1}^{N}\mathbb{E}\left \| \theta _{i}^{t,q}- \Theta _{i}^{t,q}
\right \|^{2}
\\
&\leq 2\mathbb{E}\left \|\nabla ^{2} \hat{f}\left (\omega ^{*}\right )
\left (\omega ^{t}-\omega ^{*}\right )-\nabla f^{t}+\nabla \hat{f}
\left (\omega ^{*}\right )\right \|^{2}+ \frac{2}{\psi ^{*}} \sum
\limits _{i=1}^{N} \sum \limits _{q\in \mathcal{A}^{t}} \mathbb{E}
\left \|\nabla f_{i}^{t,q}-\nabla F_{i}^{t,q}\right \|^{2}
+2\sum \limits _{q\in \mathcal{A}-\mathcal{A}^{t}}\frac{1}{N}
\sum \limits _{i =1}^{N}\mathbb{E}\left \| \theta _{i}^{t,q}- \Theta _{i}^{t,q}
\right \|^{2}
\\
&\leq 2\mathbb{E}\left \|\nabla ^{2} \hat{f}\left (\omega ^{*}\right )
\left (\omega ^{t}-\omega ^{*}\right )-\nabla f^{t}+\nabla \hat{f}
\left (\omega ^{*}\right )\right \|^{2}+ \frac{2}{\psi ^{*}}\sum
\limits _{i=1}^{N} \mathbb{E}\left \|\nabla f_{i}^{t}-\nabla F_{i}^{t}
\right \|^{2}
+2\frac{1}{N}
\sum \limits _{i =1}^{N}\mathbb{E}\left \|\nabla f_{i}^{t-\gamma^{t}_i}-\nabla F_{i}^{t-
\gamma^{t}_i}\right \|^{2}
\\
&\leq 2 \mathbb{E}\left \|\nabla ^{2} \hat{f}\left (\omega ^{*}
\right )\left (\omega ^{t}-\omega ^{*}\right )-\nabla f^{t}+\nabla
\hat{f}\left (\omega ^{*}\right )\right \|^{2}
+\frac{2}{\psi ^{*}}\sum \limits _{i=1}^{N}\mathbb{E}\left \|
\nabla f_{i}\left (\omega _{i}^{t}\right ) \odot \tau _{i}^{t}-
\nabla F_{i}\left (\omega _{i}^{t}\right ) \odot \tau _{i}^{t}\right
\|^{2}
\\
&~~~~+2\frac{1}{N}
\sum \limits _{i =1}^{N}\mathbb{E}\left \|\nabla f_{i}\left (\omega _{i}^{t-\gamma^{t}_i}
\right ) \odot \tau _{i}^{t-\gamma^{t}_i}-\nabla F_{i}\left (\omega _{i}^{t-
\gamma^{t}_i}\right ) \odot \tau _{i}^{t-\gamma^{t}_i}\right \|^{2}
\\
&\leq 2 \mathbb{E}\left \|\nabla ^{2} \hat{f}\left (\omega ^{*}
\right )\left (\omega ^{t}-\omega ^{*}\right )-\nabla f^{t}+\nabla
\hat{f}\left (\omega ^{*}\right )\right \|^{2}
+\frac{2}{\psi ^{*}} \sum \limits _{i=1}^{N}\mathbb{E}\left \|
\left (\nabla f_{i}\left (\omega _{i}^{t}\right )-\nabla F_{i}\left (
\omega _{i}^{t}\right )\right ) \odot \tau _{i}^{t}\right \|^{2}
\\
&~~~~+2\frac{1}{N}
\sum \limits _{i =1}^{N}\mathbb{E}\left \|\left (\nabla f_{i}\left (\omega _{i}^{t-
\gamma^{t}_i}\right ) -\nabla F_{i}\left (\omega _{i}^{t-\gamma^{t}_i}
\right )\right ) \odot \tau _{i}^{t-\gamma^{t}_i}\right \|^{2}
\\
&\overset{\text{Assumption~\ref{A-variance}}(i)}{\leq} 2\mathbb{E}
\left \|\nabla ^{2} \hat{f}\left (\omega ^{*}\right )\left (\omega ^{t}-
\omega ^{*}\right )-\nabla f^{t}+\nabla \hat{f}\left (\omega ^{*}
\right )\right \|^{2}+\frac{2N}{\psi ^{*}} \Lambda ^{2}+2 \Lambda ^{2}.\qedhere
\end{align*}}
\end{proof}

\begin{proof}[Proof ofLemma~\ref{le-pruning_error_VTEX1}]
\begin{align*}
&~~~~~ \mathbb{E}\left \|\nabla f^{t}-\nabla \hat{f}\left (\omega ^{*}
\right )-\nabla ^{2} \hat{f}\left (\omega ^{*}\right )\left (\omega ^{t}-
\omega ^{*}\right )\right \|^{2}
\\
& =\mathbb{E}\left \|\nabla f^{t}-\nabla \hat{f}\left (\omega ^{t}
\right )+\nabla \hat{f}\left (\omega ^{t}\right )-\nabla \hat{f}
\left (\omega ^{*}\right )-\nabla ^{2} \hat{f}\left (\omega ^{*}
\right )\left (\omega ^{t}-\omega ^{*}\right )\right \|^{2}
\\
& =\mathbb{E}\left \|
\underbrace{\left (\nabla f^{t}-\nabla \hat{f}\left (\omega ^{t}\right )\right )}_{A_{1}}+
\underbrace{\nabla \hat{f}\left (\omega ^{t}\right )-\nabla \hat{f}\left (\omega ^{*}\right )-\nabla ^{2} \hat{f}\left (\omega ^{*}\right )\left (\omega ^{t}-\omega ^{*}\right )}_{A_{2}}
\right \|^{2}
\\
& =2\mathbb{E}\left \|A_{1}\right \|^{2}+2\mathbb{E}\left \|A_{2}
\right \|^{2}.
\end{align*}

Next we bound $\mathbb{E}\left \|A_{1}\right \|^{2}$ and
$\mathbb{E}\left \|A_{2}\right \|^{2}$ respectively.

\textit{\textbf{$1$: Bounding $\mathbb{E}\left \|A_{1}\right \|^{2}$}}
%
%
\begin{align}
\label{eq-bound_a1_VTEX1}
&~~~~~\mathbb{E}\left \|A_{1}\right \|^{2}
\nonumber
\\
&=\mathbb{E}\left \|\nabla f^{t}-\nabla \hat{f}\left (\omega ^{t}
\right )\right \|^{2}
\nonumber
\\
&=\sum \limits _{q\in \mathcal{A}^{t}}\mathbb{E}\left \|
\frac{1}{\left |\mathcal{C}^{t,q}\right |}\sum \limits _{i \in
\mathcal{C}^{t,q} } (\nabla f_{i}^{t,q}-\nabla f_{i}^{q}(\omega ^{t}))
\right \|^{2} +\sum \limits _{q\in \mathcal{A}-\mathcal{A}^{t}}
\mathbb{E}\left \|\frac{1}{N}\sum \limits _{i =1}^{N} (\nabla \theta _{i}^{t,q}-
\nabla f_{i}^{q}(\omega ^{t}))\right \|^{2}
\nonumber
\\
&=\sum \limits _{q\in \mathcal{A}^{t}}\mathbb{E}\left \|
\frac{1}{\left |\mathcal{C}^{t,q}\right |}\sum \limits _{i \in
\mathcal{C}^{t,q} } (\nabla f_{i}^{q}(\omega _{i}^{t})-\nabla f_{i}^{q}(
\omega ^{t}))\right \|^{2} +\sum \limits _{q\in \mathcal{A}-
\mathcal{A}^{t}}\mathbb{E}\left \|\frac{1}{N}\sum \limits _{i =1}^{N} (
\nabla \theta _{i}^{t,q}-\nabla f_{i}^{q}(\omega ^{t}))\right \|^{2}
\nonumber
\\
&\leq\sum \limits _{q\in \mathcal{A}^{t}}
\frac{1}{\left |\mathcal{C}^{t,q}\right |} \sum \limits _{i \in
\mathcal{C}^{t,q} } \mathbb{E}\left \|\nabla f_{i}^{q}(\omega _{i}^{t})-
\nabla f_{i}^{q}(\omega ^{t})\right \|^{2} +\sum \limits _{q\in
\mathcal{A}-\mathcal{A}^{t}}\mathbb{E}\left \|\frac{1}{N}\sum
\limits _{i =1}^{N} (\nabla \theta _{i}^{t,q}-\nabla f_{i}^{q}(
\omega ^{t}))\right \|^{2}
\nonumber
\\
&\leq
\underbrace{\frac{1}{\psi ^{*}} \sum \limits _{i=1}^{N} \sum \limits _{q\in \mathcal{A}^{t}} \mathbb{E}\left \|\nabla f_{i}^{q}(\omega _{i}^{t})-\nabla f_{i}^{q}(\omega ^{t})\right \|^{2}}_{B_{1}}
+
\underbrace{\sum \limits _{q\in \mathcal{A}-\mathcal{A}^{t}}\mathbb{E}\left \|\frac{1}{N}\sum \limits _{i =1}^{N} (\nabla \theta _{i}^{t,q}-\nabla f_{i}^{q}(\omega ^{t}))\right \|^{2}}_{B_{2}}.
\end{align}

\textit{\textbf{$1.1$: Bounding $B_{1}$}}
\begin{equation*}
\begin{aligned}
B_{1}&=\frac{1}{\psi ^{*}} \sum \limits _{i=1}^{N} \sum \limits _{q
\in \mathcal{A}^{t}} \mathbb{E}\left \|\nabla f_{i}^{q}(\omega _{i}^{t})-
\nabla f_{i}^{q}(\omega ^{t})\right \|^{2}
\leq \frac{1}{\psi ^{*}} \sum \limits _{i=1}^{N} \mathbb{E}\left \|
\nabla f_{i}(\omega _{i}^{t})-\nabla f_{i}(\omega ^{t})\right \|^{2}.
\end{aligned}
\end{equation*}

According to Assumption~\ref{A-Lipschitz}, function $f_{i}(\omega )$ has
a $\iota _{g}$-Lipschitz gradient for $\forall i\in [N]$. Then it is obtained
that
%
%
\begin{equation}
\label{eq-Lg}
\left \|\nabla f_{i}\left (\omega _{i}^{t}\right )-\nabla f_{i}\left (
\omega ^{t}\right )\right \|\leq \iota _{g}\left \| \omega _{i}^{t}-
\omega ^{t}\right \|.
\end{equation}

According to Assumption~\ref{A-prun}, it holds that
$\Vert \omega ^{t}-\omega _{i}^{t} \Vert ^{2} \leq \delta ^{2}$ for
$\forall i\in [N]$.

Combining \myref{eq-Lg}, we then get that:
%
%
\begin{equation}
\label{eq-both}
\left \|\nabla f_{i}\left (\omega _{i}^{t}\right )-\nabla f_{i}\left (
\omega ^{t}\right )\right \|^{2}\leq \iota _{g}^{2}\left \| \omega _{i}^{t}-
\omega ^{t}\right \|^{2}\leq \iota _{g}^{2}\delta ^{2}.
\end{equation}

Thus $B_{1}$ is bounded as follows.
%
%
\begin{equation}
\label{bound-B_1}
B_{1}\leq \frac{1}{\psi ^{*}} \sum \limits _{i=1}^{N} \mathbb{E}
\left \|\nabla f_{i}(\omega _{i}^{t})-\nabla f_{i}(\omega ^{t})
\right \|^{2}\leq \frac{N}{\psi ^{*}} \iota _{g}^{2}\delta ^{2}.
\end{equation}

\textit{\textbf{$1.2$: Bounding $B_{2}$}}
%
%
\begin{align}
\label{eq-bound_b2_VTEX1}
B_{2}
\nonumber
&=\sum \limits _{q\in \mathcal{A}-\mathcal{A}^{t}}\mathbb{E}\left \|
\frac{1}{N}\sum \limits _{i =1}^{N} (\nabla \theta _{i}^{t,q}-\nabla f_{i}^{q}(
\omega ^{t}))\right \|^{2}
\nonumber
\\
&\leq \sum \limits _{q\in \mathcal{A}-\mathcal{A}^{t}}\mathbb{E}
\left \|\frac{1}{N}\sum \limits _{i =1}^{N} (\nabla f_{i}^{q}(\omega _{i}^{t-
\gamma^{t}_i})-\nabla f_{i}^{q}(\omega ^{t}))\right \|^{2}
\nonumber
\\
&= \sum \limits _{q\in \mathcal{A}-\mathcal{A}^{t}}\mathbb{E}
\left \|\frac{1}{N}\sum \limits _{i =1}^{N} (\nabla f_{i}^{q}(\omega _{i}^{t-
\gamma^{t}_i})-\nabla f_{i}^{q}(\omega ^{t-\gamma^{t}_i})+\nabla f_{i}^{q}(
\omega ^{t-\gamma^{t}_i})-\nabla f_{i}^{q}(\omega ^{t}))\right \|^{2}
\nonumber
\\
&\leq \frac{1}{N}\sum \limits _{i =1}^{N} \sum \limits _{q\in
\mathcal{A}-\mathcal{A}^{t}}\mathbb{E}\left \|\nabla f_{i}^{q}(
\omega _{i}^{t-\gamma^{t}_i})-\nabla f_{i}^{q}(\omega ^{t-\gamma^{t}_i})+
\nabla f_{i}^{q}(\omega ^{t-\gamma^{t}_i})-\nabla f_{i}^{q}(\omega ^{t})
\right \|^{2}
\nonumber
\\
&\leq \frac{2}{N}\sum \limits _{i =1}^{N} \sum \limits _{q\in
\mathcal{A}-\mathcal{A}^{t}}\mathbb{E}\left \|\nabla f_{i}^{q}(
\omega _{i}^{t-\gamma ^{t}_i})-\nabla f_{i}^{q}(\omega ^{t-\gamma ^{t}_i})
\right \|^{2}
\nonumber
+\frac{2}{N}\sum \limits _{i =1}^{N} \sum \limits _{q\in
\mathcal{A}-\mathcal{A}^{t}}\mathbb{E}\left \|\nabla f_{i}^{q}(
\omega ^{t-\gamma ^{t}_i})-\nabla f_{i}^{q}(\omega ^{t})\right \|^{2}
\nonumber
\\
&\leq
\underbrace{\frac{2}{N}\sum \limits _{i =1}^{N} \mathbb{E}\left \|\nabla f_{i}(\omega _{i}^{t-\gamma ^{t}_i})-\nabla f_{i}(\omega ^{t-\gamma ^{t}_i})\right \|^{2}}_{C_{1}}+\underbrace{ \frac{2}{N}\sum \limits _{i =1}^{N} \mathbb{E}\left \|\nabla f_{i}(
\omega ^{t-\gamma ^{t}_i})-\nabla f_{i}(\omega ^{t})\right \|^{2}}_{C_{2}}.
\end{align}

\textit{\textbf{$1.2.1$: Bounding $C_{1}$}}

By \myref{eq-both}, $C_{1}$ is bounded as follows.
%
%
\begin{equation}
\label{bound-C_1}
C_{1}\leq \frac{2}{N}\sum \limits _{i =1}^{N} \mathbb{E}\left \|
\nabla f_{i}(\omega _{i}^{t-\gamma ^{t}})-\nabla f_{i}(\omega ^{t-
\gamma ^{t}})\right \|^{2}\leq 2 \iota _{g}^{2}\delta ^{2}.
\end{equation}

\textit{\textbf{$1.2.2$: Bounding $C_{2}$}}
%
%
\begin{equation}
\label{eq-c}
\begin{aligned}
C_{2}&=\frac{2}{N}\sum \limits _{i =1}^{N} \mathbb{E}\left \|\nabla f_{i}(
\omega ^{t-\gamma ^{t}_i})-\nabla f_{i}(\omega ^{t})\right \|^{2}
\\
&\overset{\myref{eq-Lg}}{\leq} \frac{2\iota _{g}^{2}}{N}\sum
\limits _{i =1}^{N} \mathbb{E}\left \|\omega ^{t-\gamma ^{t}_i}-\omega ^{t}
\right \|^{2}
\\
&\leq 2\iota _{g}^{2} \mathbb{E}\left \|\omega ^{t-\gamma ^{t}}-\omega ^{t}
\right \|^{2}
\\
&\leq 2\iota _{g}^{2}\gamma ^{t}
\sum \limits _{l=0}^{\gamma ^{t}-1}\mathbb{E}\left \|\omega ^{t-l}-
\omega ^{t-(l+1)}\right \|^{2}.
\end{aligned}
%
\end{equation}

Next, we bound
$\mathbb{E}\left \|\omega ^{t-l}-\omega ^{t-(l+1)}\right \|^{2}$. Recalling the lobal parameter update rule
$\omega ^{t+1} = \omega ^{t} - [\mathbf{\Pi}]_{\mu}^{-1} \nabla F^{t}$, we have
%
%
\begin{equation}
\label{eq-ctotal2}
\begin{aligned}
&~~~~~\mathbb{E}\left \|\omega ^{t-l}-\omega ^{t-(l+1)}\right \|^{2}
\\
&=\mathbb{E}\left \|[\mathbf{\Pi}]_{\mu}^{-1} \nabla F^{t-(l+1)}
\right \|^{2}
\\
&\overset{\text{Assumption~\ref{A-strongly_convex_VTEX1}}}{\leq}\frac{1}{\mu ^{2}}
\mathbb{E}\left \|\nabla F^{t-(l+1)}\right \|^{2}
\\
&\leq \frac{1}{\mu ^{2}} \left (\sum \limits _{q\in \mathcal{A}^{t-(l+1)}}
\mathbb{E}\left \|\nabla F^{{t-(l+1)},q}\right \|^{2} + \sum \limits _{q
\in \mathcal{A}-\mathcal{A}^{t-(l+1)}}\mathbb{E}\left \|\nabla F^{{t-(l+1)},q}
\right \|^{2}\right ).
\end{aligned}
%
\end{equation}

$\sum \limits _{q\in \mathcal{A}^{t-(l+1)}}\mathbb{E}\left \|\nabla F^{{t-(l+1)},q}
\right \|^{2}$ can be further bounded as follows.
\begin{equation*}
\begin{aligned}
&~~~~~\sum \limits _{q\in \mathcal{A}^{t-(l+1)}}\mathbb{E}\left \|
\nabla F^{{t-(l+1)},q}\right \|^{2}
\\
&\leq \sum \limits _{q\in \mathcal{A}^{t-(l+1)}}\mathbb{E}\left \|
\frac{1}{\left |\mathcal{C}^{{t-(l+1)},q}\right |}\sum \limits _{i
\in \mathcal{C}^{{t-(l+1)},q} }\nabla F_{i}^{{t-(l+1)},q}\right \|^{2}
\\
&\leq \frac{1}{\left |\mathcal{C}^{{t-(l+1)},q}\right |}\sum \limits _{i
\in \mathcal{C}^{{t-(l+1)},q} }\mathbb{E}\left \|\nabla F_{i}^{t-(l+1)}
\right \|^{2}
\\
&\overset{\text{Assumption~\ref{A-variance}}}{\leq}\frac{1}{\left |\mathcal{C}^{{t-(l+1)},q}\right |}\sum \limits _{i
\in \mathcal{C}^{{t-(l+1)},q} }\left(\mathbb{E}\left \|\nabla f_{i}^{t-(l+1)}
\right \|^{2}+\Lambda ^{2}\right)
\\
&\leq \frac{1}{\left |\mathcal{C}^{{t-(l+1)},q}\right |}\sum \limits _{i
\in \mathcal{C}^{{t-(l+1)},q} }\left(\mathbb{E}\left \|\nabla f_{i}(\omega _{i}^{t-(l+1)})
\right \|^{2}+\Lambda ^{2}\right).
\end{aligned}
\end{equation*}

According to Assumption~\ref{A-Lipschitz}, function $f_{i}(\omega )$ is
$\iota $-Lipschitz for $\forall i\in [N]$. Then it is obtained that:
%
%
\begin{equation}
\label{eq-l_lip_VTEX1}
\left \|\nabla f_{i}\left (\omega _{i}^{t}\right )\right \| \leq L.
\end{equation}

And we have:
%
%
\begin{equation}
\label{eq-c21}
\sum \limits _{q\in \mathcal{A}^{t-(l+1)}}\mathbb{E}\left \|\nabla F^{{t-(l+1)},q}
\right \|^{2}\overset{\text{Assumption~\ref{A-variance}}(i)}{\leq}L^{2}+\Lambda ^{2}.
\end{equation}

We then further bound
$ \sum \limits _{q\in \mathcal{A}-\mathcal{A}^{t-(l+1)}}\mathbb{E}
\left \|\nabla F^{{t-(l+1)},q}\right \|^{2}$.
\begin{equation*}
\begin{aligned}
&~~~~~ \sum \limits _{q\in \mathcal{A}-\mathcal{A}^{t-(l+1)}}
\mathbb{E}\left \|\nabla F^{{t-(l+1)},q}\right \|^{2}
\leq \frac{1}{N}\sum \limits _{i =1 }^{N}\mathbb{E}\left \|\nabla
\Theta _{i}^{t-(l+1)}\right \|^{2}
\leq \frac{1}{N}\sum \limits _{i =1 }^{N}\mathbb{E}\left \|\nabla F_{i}^{t-(l+1)-
\gamma_i^{t-(l+1)}}\right \|^{2}\\
&
\overset{\text{Assumption~\ref{A-variance}}(i)}{\leq} \frac{1}{N}\sum \limits _{i =1 }^{N}\left(\mathbb{E}\left \|\nabla f_{i}^{t-(l+1)-
\gamma_i^{t-(l+1)}}\right \|^{2}+\Lambda ^{2}\right)
\leq \frac{1}{N}\sum \limits _{i =1 }^{N}\left(\mathbb{E}\left \|\nabla f_{i}(
\omega _{i}^{t-(l+1)-\gamma_i^{t-(l+1)}})\right \|^{2}+\Lambda ^{2}\right).
\end{aligned}
\end{equation*}

According to \myref{eq-l_lip_VTEX1}, we have:
%
%
\begin{equation}
\label{eq-c22}
\sum \limits _{q\in \mathcal{A}-\mathcal{A}^{t-(l+1)}}\mathbb{E}
\left \|\nabla F^{{t-(l+1)},q}\right \|^{2}\overset{\text{Assumption~\ref{A-variance}}(i)}{\leq} L^{2}+\Lambda ^{2}.
\end{equation}

Plugging \myref{eq-c21} and \myref{eq-c22} into \myref{eq-ctotal2},
we have
\begin{equation*}
\begin{aligned}
&~~~~~\mathbb{E}\left \|\omega ^{t-l}-\omega ^{t-(l+1)}\right \|^{2}
\leq \frac{1}{\mu ^{2}} \left (\sum \limits _{q\in \mathcal{A}^{t-(l+1)}}
\mathbb{E}\left \|\nabla F^{{t-(l+1)},q}\right \|^{2} + \sum \limits _{q
\in \mathcal{A}-\mathcal{A}^{t-(l+1)}}\mathbb{E}\left \|\nabla F^{{t-(l+1)},q}
\right \|^{2}\right )
\leq \frac{2(L^{2}+\Lambda ^{2})}{\mu ^{2}}.
\end{aligned}
\end{equation*}

Taking above result back to \myref{eq-c}, we obtain that
%
%

\begin{equation}
\label{eq-c2final}
\begin{aligned}
C_{2}&\leq \frac{2\iota _{g}^{2}\gamma ^{t}}{N}\sum \limits _{i =1}^{N}
\sum \limits _{l=0}^{\gamma ^{t}-1}\mathbb{E}\left \|\omega ^{t-l}-
\omega ^{t-(l+1)}\right \|^{2}
\leq \frac{4(L^{2}+\Lambda ^{2})\iota _{g}^{2}(\gamma ^{t})^{2}}{\mu ^{2}}.
\end{aligned}
%
\end{equation}

Plugging \myref{bound-C_1} and \myref{eq-c2final} into \myref{eq-bound_b2_VTEX1},
we have
%
%
{\footnotesize
\begin{equation}
\label{bound_b2_VTEX1}
\begin{aligned}
B_{2}&=\sum \limits _{q\in \mathcal{A}-\mathcal{A}^{t}}\mathbb{E}
\left \|\frac{1}{N}\sum \limits _{i =1}^{N} (\nabla \theta _{i}^{t,q}-
\nabla f_{i}^{q}(\omega ^{t}))\right \|^{2}
\leq
\underbrace{\frac{2}{N}\sum \limits _{i =1}^{N} \mathbb{E}\left \|\nabla f_{i}(\omega _{i}^{t-\gamma ^{t}})-\nabla f_{i}(\omega ^{t-\gamma ^{t}})\right \|^{2}}_{C_{1}}
+
\underbrace{ \frac{2}{N}\sum \limits _{i =1}^{N} \sum \limits _{q\in \mathcal{A}-\mathcal{A}^{t}}\mathbb{E}\left \|\nabla f_{i}^{q}(\omega ^{t-\gamma ^{t}})-\nabla f_{i}^{q}(\omega ^{t})\right \|^{2}}_{C_{2}}
\\
&\leq 2 \iota _{g}^{2}\delta ^{2}+
\frac{4(L^{2}+\Lambda ^{2})\iota _{g}^{2}(\gamma ^{t})^{2}}{\mu ^{2}}.
\end{aligned}
%
\end{equation}}

By \myref{bound-B_1}, \myref{bound_b2_VTEX1} and \ref{eq-bound_a1_VTEX1},
we get:
%
%
\begin{equation}
\label{eq-bound_a1_final_VTEX1}
\begin{aligned}
&~~~~~\mathbb{E}\left \|A_{1}\right \|^{2}
\\
&\leq
\underbrace{\frac{1}{\psi ^{*}} \sum \limits _{i=1}^{N} \sum \limits _{q\in \mathcal{A}^{t}} \mathbb{E}\left \|\nabla f_{i}^{q}(\omega _{i}^{t})-\nabla f_{i}^{q}(\omega ^{t})\right \|^{2}}_{B_{1}}
+
\underbrace{\sum \limits _{q\in \mathcal{A}-\mathcal{A}^{t}}\mathbb{E}\left \|\frac{1}{N}\sum \limits _{i =1}^{N} (\nabla \theta _{i}^{t,q}-\nabla f_{i}^{q}(\omega ^{t}))\right \|^{2}}_{B_{2}}
\\
&\leq \frac{N}{\psi ^{*}} \iota _{g}^{2}\delta ^{2}+ 2 \iota _{g}^{2}
\delta ^{2}+\frac{4(L^{2}+\Lambda ^{2})\iota _{g}^{2}(\gamma ^{t})^{2}}{\mu ^{2}}.
\end{aligned}
%
\end{equation}

\textit{\textbf{$2$: Bounding $\mathbb{E}\left \|A_{2}\right \|^{2}$}}

Recalling Assumption~\ref{A-Lipschitz}, which assumes that each function $f_{i}(\omega )$ has an $\iota _{h}$-Lipschitz Hessian for $\forall i\in [N]$, we can apply the triangle inequality to obtain:
\[
\|\nabla^2 \hat{f}(\omega ^{t}) - \nabla^2 \hat{f}(\omega ^{*})\| \leq \frac{1}{N} \sum_{i=1}^N \|\nabla^2 f_i(\omega ^{t}) - \nabla^2 f_i(\omega ^{*})\|\leq \frac{1}{N} \sum_{i=1}^N \iota _{h} \|\omega ^{t} - \omega ^{*}\| =\iota _{h} \|\omega ^{t} - \omega ^{*}\|.
\]

This shows that $\hat{f}(\cdot)$ also has an $\iota_h$-Lipschitz Hessian. Furthermore, by Lemma 1 in~\cite{4DBLP:journals/mp/NesterovP06}, the following holds:
%
%
\begin{equation}
\label{eq-global_Hessian_Lip_VTEX1}
\left \|\nabla
\hat{f}\left (\omega ^{t}\right )-\nabla \hat{f}\left (\omega ^{*}
\right )-\nabla ^{2} \hat{f}\left (\omega ^{*}\right )\left (\omega ^{t}-
\omega ^{*}\right )\right \| \leq \frac{\iota _{h}}{2}
\lVert \omega ^{t}-\omega ^{*}\rVert ^{2}.
\end{equation}

Using this, we can derive the upper bound for $\mathbb{E}\lVert A_{2} \rVert ^{2}$:
%
%
\begin{equation}
\label{bound-A_2}
\begin{aligned}
\mathbb{E}\left \|A_{2}\right \|^{2}& = \mathbb{E}\left \|\nabla
\hat{f}\left (\omega ^{t}\right )-\nabla \hat{f}\left (\omega ^{*}
\right )-\nabla ^{2} \hat{f}\left (\omega ^{*}\right )\left (\omega ^{t}-
\omega ^{*}\right )\right \|^{2}
\leq \frac{\iota _{h}^{2}}{4}\mathbb{E}\left \|\omega ^{t}-
\omega ^{*}\right \|^{4}.
\end{aligned}
%
\end{equation}

Finally, by combining \myref{eq-bound_a1_final_VTEX1} and \myref{bound-A_2}, we
obtain the upper bound of
$\mathbb{E}\left \|\nabla f^{t}-\nabla \hat{f}\left (\omega ^{*}
\right )-\nabla ^{2} \hat{f}\left (\omega ^{*}\right )\left (\omega ^{t}-
\omega ^{*}\right )\right \|^{2}$.
\begin{equation*}
\begin{aligned}
&~~~~~\mathbb{E}\left \|\nabla f^{t}-\nabla \hat{f}\left (\omega ^{*}
\right )-\nabla ^{2} \hat{f}\left (\omega ^{*}\right )\left (\omega ^{t}-
\omega ^{*}\right )\right \|^{2}
\\
& =2\mathbb{E}\left \|A_{1}\right \|^{2}+ 2\mathbb{E}\left \|A_{2}
\right \|^{2}
\\
&\leq 2\left (\frac{N}{\psi ^{*}} \iota _{g}^{2}\delta ^{2}+ 2 \iota _{g}^{2}
\delta ^{2}+\frac{4(L^{2}+\Lambda ^{2})\iota _{g}^{2}(\gamma ^{t})^{2}}{\mu ^{2}}+
\frac{\iota _{h}^{2}}{4}\mathbb{E}\left \|\omega ^{t}-
\omega ^{*}\right \|^{4}\right )
\\
&\leq \frac{N}{\psi ^{*}}\left (2\iota _{g}^{2}\delta ^{2}\right )+
(\gamma ^{t})^{2} \frac{8(L^{2}+\Lambda ^{2})\iota _{g}^{2}}{\mu ^{2}}+4\iota _{g}^{2}
\delta ^{2}+ \frac{\iota _{h}^{2}}{2}\mathbb{E}\left \|
\omega ^{t}-\omega ^{*}\right \|^{4}.
\end{aligned}
\end{equation*}
\end{proof}

\begin{proof}[Proof of Theorem~\ref{theorem-rate}]
%
%
\begin{equation}
\label{eq-theo0}
\begin{aligned}
&~~~~ \mathbb{E}\left \|\omega ^{t+1}-\omega ^{*}\right \|^{2}
\\
& =\mathbb{E}\left \|\omega ^{t}-\omega ^{*}-[\mathbf{\Pi}]_{\mu}^{-1}
\nabla F^{t}\right \|^{2}
\\
& \leq \mathbb{E}\left \|[\mathbf{\Pi}]_{\mu}^{-1}\right \|^{2}\left
\|[\mathbf{\Pi}]_{\mu}\left (\omega ^{t}-\omega ^{*}\right )-\nabla F^{t}
\right \|^{2}
\\
&\overset{\text{Assumption~\ref{A-strongly_convex_VTEX1}}}{\leq}
\frac{1}{\mu ^{2}}\mathbb{E}\left \|\left ([\mathbf{\Pi}]_{\mu}-
\nabla ^{2} \hat{f}\left (\omega ^{*}\right ) \right )\left ( \omega ^{t}-
\omega ^{*}\right )+\nabla ^{2} \hat{f}\left (\omega ^{*}\right )
\left (\omega ^{t}-\omega ^{*}\right )-\nabla F^{t}+\nabla \hat{f}
\left (\omega ^{*}\right )\right \|^{2}
\\
& \leq \frac{2}{\mu ^{2}}\left (\mathbb{E}\left \|\left ([
\mathbf{\Pi}]_{\mu}-\nabla ^{2} \hat{f}\left (\omega ^{*}\right )
\right )\left (\omega ^{t}-\omega ^{*}\right )\right \|^{2}+
\mathbb{E}\left \|\nabla ^{2} \hat{f}\left (\omega ^{*}\right )\left (
\omega ^{t}-\omega ^{*}\right )-\nabla F^{t}+\nabla \hat{f}\left (
\omega ^{*}\right )\right \|^{2}\right )
\\
& \leq \frac{2}{\mu ^{2}}\left (\mathbb{E}\left \|\omega ^{t}-\omega ^{*}
\right \|^{2}\mathbb{E}\left \| [\mathbf{\Pi}]_{\mu}-\nabla ^{2}
\hat{f}\left (\omega ^{*}\right )\right \|^{2}+\mathbb{E}\left \|
\nabla ^{2} \hat{f}\left (\omega ^{*}\right )\left (\omega ^{t}-
\omega ^{*}\right )-\nabla F^{t}+\nabla \hat{f}\left (\omega ^{*}
\right )\right \|^{2}\right )
\\
&\overset{\text{Lemma~\ref{le-projection}}}{\leq} \frac{2}{\mu ^{2}}\left (
\mathbb{E}\left \|\omega ^{t}-\omega ^{*}\right \|^{2}\mathbb{E}
\left \| \mathbf{\Pi}-\nabla ^{2} \hat{f}\left (\omega ^{*}\right )
\right \|^{2}+\mathbb{E}\left \|\nabla ^{2} \hat{f}\left (\omega ^{*}
\right )\left (\omega ^{t}-\omega ^{*}\right )-\nabla F^{t}+\nabla
\hat{f}\left (\omega ^{*}\right )\right \|^{2}\right )
\\
&\overset{\text{Lemma~\ref{le-stochastic_hessian_error_VTEX1}}}{\leq}
\frac{2}{\mu ^{2}}\left [\mathbb{E}\left \|\omega ^{t}-\omega ^{*}
\right \|^{2}\left (2\mathbb{E}\left \|\nabla ^{2} \hat{f}\left (
\omega ^{0}\right )-\nabla ^{2} \hat{f}\left (\omega ^{*}\right )
\right \|_{F}^{2}+2\kappa ^{2}\right )\right .
\left .+\mathbb{E}\left \|\nabla ^{2} \hat{f}\left (
\omega ^{*}\right )\left (\omega ^{t}-\omega ^{*}\right )-\nabla F^{t}+
\nabla \hat{f}\left (\omega ^{*}\right )\right \|^{2}\right ]
\\
&\overset{\text{Lemma~\ref{le-stochastic_gradient_error_VTEX1}}}{\leq}
\frac{2}{\mu ^{2}}\left [\mathbb{E}\left \|\omega ^{t}-\omega ^{*}
\right \|^{2}\left (2\mathbb{E}\left \|\nabla ^{2} \hat{f}\left (
\omega ^{0}\right )-\nabla ^{2} \hat{f}\left (\omega ^{*}\right )
\right \|_{F}^{2}+2\kappa ^{2}\right )\right .
\\
&~~~~~~~~~~\left .+2\mathbb{E}\left \|\nabla ^{2} \hat{f}\left (
\omega ^{*}\right )\left (\omega ^{t}-\omega ^{*}\right )-\nabla f^{t}+
\nabla \hat{f}\left (\omega ^{*}\right )\right \|^{2}+
\frac{2N}{\psi ^{*}} \Lambda ^{2}+2 \Lambda ^{2}\right ]
\\
&\overset{\text{Lemma~\ref{le-pruning_error_VTEX1}}}{\leq} \frac{2}{\mu ^{2}}
\Bigg [\mathbb{E}\left \|\omega ^{t}-\omega ^{*}\right \|^{2}\left (2
\mathbb{E}\left \|\nabla ^{2} \hat{f}\left (\omega ^{0}\right )-
\nabla ^{2} \hat{f}\left (\omega ^{*}\right )\right \|_{F}^{2}+2
\kappa ^{2}\right )
\\
&~~~~~~~~~~+2 \left ( \frac{N}{\psi ^{*}}\left (2\iota _{g}^{2}
\delta ^{2}\right )+(\gamma ^{t})^{2}
\frac{8(L^{2}+\Lambda ^{2})\iota _{g}^{2}}{\mu ^{2}}+4\iota _{g}^{2}\delta ^{2}+
\frac{\iota _{h}^{2}}{2}\mathbb{E}\left \|\omega ^{t}-
\omega ^{*}\right \|^{4}\right )+\frac{2N}{\psi ^{*}} \Lambda ^{2}+2
\Lambda ^{2}\Bigg]
\\
& \leq \frac{2}{\mu ^{2}}\Bigg[\mathbb{E}\left \|\omega ^{t}-\omega ^{*}
\right \|^{2}\left (2\mathbb{E}\left \|\nabla ^{2} \hat{f}\left (
\omega ^{0}\right )-\nabla ^{2} \hat{f}\left (\omega ^{*}\right )
\right \|_{F}^{2}+2\kappa ^{2}\right )
\\
&~~~~~~~~~~+2\Bigg( \frac{N}{\psi ^{*}}\left (2\iota _{g}^{2}\delta ^{2}+
\Lambda ^{2}\right )+(\gamma ^{t})^{2}
\frac{8(L^{2}+\Lambda ^{2})\iota _{g}^{2}}{\mu ^{2}}+4\iota _{g}^{2}\delta ^{2}+
\Lambda ^{2}+ \frac{\iota _{h}^{2}}{2}\mathbb{E}\left \|
\omega ^{t}-\omega ^{*}\right \|^{4}\Bigg)\Bigg].
\end{aligned}
%
\end{equation}

Since $\frac{b-\sqrt{\rho}}{2a}\le\|\omega^0-\omega^*\|\le\min\left\{\frac{\mu}{4\iota_h}, \frac{b+\sqrt{\rho}}{2a}\right\}\le\frac{b+\sqrt{\rho}}{2a}$, it follows that $\frac{b-\sqrt{\rho}}{2a}\le \|\omega^0-\omega^*\|\le \frac{b+\sqrt{\rho}}{2a}$. This implies that
\begin{equation*}
a\|\omega^0-\omega^*\|^2-b\|\omega^0-\omega^*\|+c\le 0.
\end{equation*}

Since $\|\omega^0-\omega^*\|>0$, we can equivalently state that
\begin{equation*}
\frac{c}{\|\omega^0-\omega^*\|}\le b-a\|\omega^0-\omega^*\|.
\end{equation*}
This shows that there exists at least one valid $\alpha>0$ such that $\alpha\ge\frac{c}{\|\omega^0-\omega^*\|}$ and $\alpha\le b-a\|\omega^0-\omega^*\|$. Using the fact that $c\le\alpha\|\omega^0-\omega^*\|$ and \myref{eq-theo0}, we further obtain:

\begin{equation}
\label{eq-theo}
	\mathbb{E}\|\omega^{t+1}-\omega^*\|^2\leq \frac{4}{\mu ^{2}} \mathbb{E}\left \|\omega ^{t}-\omega ^{*}
\right \|^{2}\left (\mathbb{E}\left \|\nabla ^{2} \hat{f}\left (
\omega ^{0}\right )-\nabla ^{2} \hat{f}\left (\omega ^{*}\right )
\right \|_{F}^{2}+\kappa ^{2}+\alpha +
\frac{\iota _{h}^{2}}{2}\mathbb{E}\| \omega ^{t}-\omega ^{*}
\|^{2}\right )
\end{equation}

Given that
$\|\omega^0-\omega^*\|\le\frac{b-\alpha}{a}$, we next prove by induction that $\mathbb{E}\|\omega^t-\omega^*\|\le\frac{b-\alpha}{a}$ holds for all $t\ge 0$. Firstly, The claim holds trivially since for the base case of $t=0$ since $\|\omega^0-\omega^*\|\le\frac{b-\alpha}{a}$ by our assumption. Next, we assume $\mathbb{E}\|\omega^t-\omega^*\|\le\frac{b-\alpha}{a}$ holds for some $t\ge 0$. We show that $\mathbb{E}\|\omega^{t+1}-\omega^*\|\le\frac{b-\alpha}{a}$ also holds. Since $\|\omega^0-\omega^*\|\le\frac{\mu}{4\iota_h}$, we have $\left \|\nabla ^{2} \hat{f}\left (\omega ^{0}\right )-\nabla ^{2}
\hat{f}\left (\omega ^{*}\right )\right \|_{F}^{2}\leq
\frac{\mu ^{2}}{16}$. Substituting this into \myref{eq-theo}, we get:
\begin{equation}
\label{eq-converge}
\begin{aligned}
&~~~~~\mathbb{E}\|\omega ^{t+1}-\omega ^{*}\|^{2}
\\
& \leq \frac{4}{\mu ^{2}} \mathbb{E}\left \|\omega ^{t}-\omega ^{*}
\right \|^{2}\left (\mathbb{E}\left \|\nabla ^{2} \hat{f}\left (
\omega ^{0}\right )-\nabla ^{2} \hat{f}\left (\omega ^{*}\right )
\right \|_{F}^{2}+\kappa ^{2}+\alpha +
\frac{\iota _{h}^{2}}{2}\mathbb{E}\| \omega ^{t}-\omega ^{*}
\|^{2}\right )
\\
&\leq \frac{4}{\mu ^{2}} \mathbb{E}\left \|\omega ^{t}-\omega ^{*}
\right \|^{2}\left (\frac{\mu ^{2}}{16}+\kappa ^{2}+\alpha +b-\alpha
\right )
\\
&\leq \frac{1}{2} \mathbb{E}\left \|\omega ^{t}-\omega ^{*}\right \|^{2}\\
&\leq \frac{b-\alpha}{a}.
\end{aligned}
\end{equation}

Thus, by induction, $\mathbb{E}\|\omega ^{t}-\omega ^{*}\|^{2} \leq \frac{b-\alpha}{a}$ holds for all $t\ge 0$. Using this and following the same steps as in the first three inqualities of \myref{eq-converge}, we conclude that the convergence rate satisfies:
$\mathbb{E}\|\omega ^{t+1}-\omega ^{*}\|^{2}\leq \frac{1}{2}
\mathbb{E}\left \|\omega ^{t}-\omega ^{*}\right \|^{2}$.
\end{proof}

\end{document}